\documentclass{article}

\usepackage{microtype}
\usepackage{subfigure}
\usepackage{booktabs} 

\usepackage{hyperref}

\usepackage[accepted]{icml2024}

\usepackage{amsmath}
\usepackage{amssymb}
\usepackage{mathtools}
\usepackage{amsthm}
\usepackage{thm-restate}

\usepackage[capitalize,noabbrev]{cleveref}

\theoremstyle{plain}
\newtheorem{theorem}{Theorem}[section]

\newtheorem{lemma}[theorem]{Lemma}
\newtheorem{corollary}[theorem]{Corollary}
\newtheorem{conjecture}[theorem]{Conjecture}
\theoremstyle{definition}
\newtheorem{definition}[theorem]{Definition}

\theoremstyle{remark}

\usepackage[textsize=tiny]{todonotes}

\usepackage{natbib}

\usepackage{bbm}

\usepackage{physics}
\usepackage{graphicx}
\usepackage{siunitx}

\usepackage{xcolor}

\newcommand{\joonas}[1]{{}}
\newcommand{\ossi}[1]{{}}

\newcommand{\calb}{\mathcal{B}}
\newcommand{\caln}{\mathcal{N}}
\newcommand{\calm}{\mathcal{M}}
\newcommand{\call}{\mathcal{L}}
\newcommand{\calx}{\mathcal{X}}
\newcommand{\AO}{\mathrm{AO}}
\newcommand{\TV}{\mathrm{TV}}
\newcommand{\KL}{\mathrm{KL}}
\newcommand{\dx}{\mathrm{d}}
\newcommand{\R}{\mathbb{R}}
\newcommand{\N}{\mathbb{N}}

\newcommand{\clip}{\mathrm{clip}}

\newcommand{\lp}{\left(}
\newcommand{\rp}{\right)}

\newcommand{\dpG}{\tilde{G}} 

\DeclareMathOperator{\Var}{Var}
\DeclareMathOperator{\Exp}{\mathbb{E}}

\newcommand{\seff}{\sigma_{\text{eff}}}

\graphicspath{{./figures/}}

\icmltitlerunning{Subsampling is not Magic: Why Large Batch Sizes Work for Differentially Private 
Stochastic Optimisation}

\begin{document}

\twocolumn[
\icmltitle{Subsampling is not Magic: Why Large Batch Sizes Work for Differentially Private 
Stochastic Optimisation}

\icmlsetsymbol{equal}{*}

\begin{icmlauthorlist}
\icmlauthor{Ossi Räisä}{equal,uh}
\icmlauthor{Joonas Jälkö}{equal,uh}
\icmlauthor{Antti Honkela}{uh}
\end{icmlauthorlist}

\icmlaffiliation{uh}{Department of Computer Science, University of Helsinki}

\icmlcorrespondingauthor{Ossi Räisä}{ossi.raisa@helsinki.fi}
\icmlcorrespondingauthor{Joonas Jälkö}{joonas.jalko@helsinki.fi}

\icmlkeywords{DP-SGD, Subsampling Amplification, Differential Privacy}

\vskip 0.3in
]

\printAffiliationsAndNotice{\icmlEqualContribution} 

\begin{abstract}
We study how the batch size affects the total gradient variance in 
differentially private  stochastic gradient descent (DP-SGD), seeking a theoretical explanation 
for the usefulness of large batch sizes. As DP-SGD is the basis of modern DP deep learning, its properties
have been widely studied, and recent works have empirically found large batch sizes
to be beneficial. However, theoretical explanations of this benefit are currently heuristic 
at best. We first observe that the total gradient variance in DP-SGD can be decomposed into 
subsampling-induced and noise-induced variances. We then prove that in the limit of an
infinite number of iterations, the effective noise-induced variance is invariant to the batch size.
The remaining subsampling-induced variance decreases with larger batch sizes,
so large batches reduce the effective total gradient variance. We confirm numerically
that the asymptotic regime is relevant in practical settings when the batch size 
is not small, and find that outside the asymptotic 
regime, the total gradient variance decreases even more with large batch sizes. 
We also find a sufficient condition that implies that large 
batch sizes similarly reduce effective DP noise variance for one iteration of DP-SGD
\footnote{Addendum: The corresponding Conjecture~\ref{conj:a-minus-b} has been proven 
by~\citet{kalininNotesSampledGaussian2024} after the publication of this work.}.
\end{abstract}

\section{Introduction}

As deep learning models are being trained on ever larger datasets, the privacy
of the subjects of these training datasets is a growing concern.
\emph{Differential privacy} (DP)~\citep{dworkCalibratingNoiseSensitivity2006} is
a property of an algorithm that formally quantifies the privacy leakage that can
result from releasing the output of the algorithm. Due to the formal guarantee
provided by DP, there is a great deal of interest in training deep learning
models with a DP variant of stochastic gradient descent
(DP-SGD)~\citep{songStochasticGradientDescent2013,
bassilyPrivateEmpiricalRisk2014,abadiDeepLearningDifferential2016}.

One of the key properties of DP is so called \emph{subsampling amplification}~\citep{liSamplingAnonymizationDifferential2012,beimelBoundsSampleComplexity2014}. 
Broadly speaking, subsampling the data before 
applying a DP mechanism adds an additional layer of protection to the data samples,
leading to stronger privacy guarantees for a fixed amount of added noise. Quantifying 
the gains from subsampling amplification~\citep{abadiDeepLearningDifferential2016,zhuPoissionSubsampledRenyi2019,KoskelaJH20,zhuOptimalAccountingDifferential2022} 
has been a crucial component in making algorithms such as DP-SGD
work under strict privacy guarantees.

In non-DP gradient descent, subsampling the gradients 
has been shown to provide better generalization on many 
occasions~\citep{hardtTrainFasterGeneralize2016,mouGeneralizationBoundsSGLD2018,kleinbergAlternativeViewWhen2018}. 
The reason for the improved generalization is that the subsampling-induced noise
allows the optimization to avoid collapsing into a local optimum early 
on~\citep{keskarLargeBatchTrainingDeep2017}. 
However, when optimizing with DP-SGD, many works have shown that the large 
batch sizes actually provide better performance compared to heavy subsampling
of the training data \cite{mcmahanLearningDifferentiallyPrivate2018, deUnlocking2022, mehtaLargeScaleTransfer2023}. 
This suggests that there are some fundamental differences in how 
the level of subsampling affects the performance of non-private and private 
SGD.

We focus on the Poisson subsampled Gaussian mechanism, which is the basis for the privacy analysis 
of DP-SGD. The privacy analysis can be done with several different 
upper bounds~\citep{abadiDeepLearningDifferential2016,zhuPoissionSubsampledRenyi2019,KoskelaJH20},
but we only look at the tight analysis that gives the smallest possible privacy bounds
for the subsampled Gaussian mechanism. The tight analysis is what modern numerical privacy 
accountants~\citep{KoskelaJH20,koskelaTightDifferentialPrivacy2021,GopiLW21,doroshenkoConnectDotsTighter2022a,alghamdiSaddlePointMethodDifferential2023} approximate.

In Poisson subsampling, each datapoint is included with probability $q$, called
the \emph{subsampling rate}, which is proportional to the expected batch size.
The total gradient variance in DP-SGD can be decomposed into two parts: the
subsampling variance and the Gaussian noise variance $\sigma^2$. A larger $q$
reduces the subsampling variance, but the effect on the noise variance is not as
clear. On one hand, a larger $q$ reduces the privacy amplification effect,
necessitating a larger $\sigma^2$, but on the other hand, an unbiased gradient
estimate must be divided by $q$, so the effective noise variance is
$\frac{\sigma^2}{q^2}$.

We study how the effective noise variance scales with the subsampling rate, making 
the following contributions:
\begin{enumerate}
    \item 
    In Section~\ref{sec:asym-linear-sigma}, we prove that in  the limit of an infinite number of iterations, there is a linear relationship 
    $\sigma = cq$ between $q$ and $\sigma$, meaning that the two effects $q$ has on
    the effective noise variance cancel each other, leaving the effective noise variance 
    invariant to the subsampling rate. This 
    means that a larger subsampling  rate always reduces the effective total gradient variance, since
    the subsampling-induced variance decreases with a larger subsampling rate.
    \item In Section~\ref{sec:no-comp}, we consider the case of a single iteration of 
    the subsampled Gaussian mechanism. We find a sufficient condition which implies 
    that large subsampling rates always reduce the effective injected DP noise variance, 
    hence also reducing the effective total gradient variance. We check this 
    condition numerically for a wide grid of hyperparameter values, and find that 
    the condition holds amongst these hyperparameters.
    \item We look at the relationship between the subsampling rate and noise standard deviation 
    empirically in Section~\ref{sec:empirical-results}, and find that the asymptotic regime 
    from our theory is reached quickly with small privacy parameters. Moreover, we find 
    that when we are not in the asymptotic regime, the effective injected DP noise variance decreases 
    even more with a large subsampling rate.
\end{enumerate}

\subsection{Related Work}\label{sec:related-work}
Several works have empirically observed that large batch sizes are useful in 
DP-SGD~\citep{mcmahanLearningDifferentiallyPrivate2018,deUnlocking2022,mehtaLargeScaleTransfer2023}.
Indeed, the linear relationship between the subsampling rate and noise standard deviation we 
study has been suggested as a heuristic rule-of-thumb in previous 
work~\citep{liLargeLanguageModels2021,sanderTANBurnScaling2023} to explain these results.

The linear relationship also appears in several works that study 
Rényi DP~\citep{mironovRenyiDifferentialPrivacy2017}
accounting of the subsampled Gaussian mechanism~\citep{
abadiDeepLearningDifferential2016,
bunComposableVersatilePrivacy2018,
mironovRenyiDifferentialPrivacy2019}, though these works do not make the connection with large 
subsampling rates. These Rényi DP-based analyses of 
the subsampled Gaussian mechanism do not provide tight 
privacy bounds anyway~\citep{zhuOptimalAccountingDifferential2022}, 
so these results do not imply that the linear relationship holds 
even asymptotically with tight accounting.

Our work focuses on Poisson subsampling, which is widely used and implemented 
in libraries like Opacus~\citep{yousefpourOpacusUserFriendlyDifferential2021a}.
There are other subsampling schemes like subsampling with replacement (WR) and 
subsampling without replacement (WOR), which take a fixed-size uniformly random subsample,
either with or without replacement.
Previous work has done tight privacy accounting for both of 
them~\citep{KoskelaJH20,zhuOptimalAccountingDifferential2022}. However, recent work 
has questioned the validity of these works for WOR 
subsampling~\citep{lebedaAvoidingPitfallsPrivacy2023}, pointing out issues which should 
be resolved before studying WOR subsampling further. WR subsampling is rarely used in practice,
and the accounting for it is much more complex~\citep{KoskelaJH20} than Poisson subsampling.

\citet{sommerPrivacyLossClasses2019} and \citet{DongRS22} prove central limit theorems
for privacy accounting, which essentially show that
the privacy loss of any sufficiently well-behaved mechanism after many compositions is asymptotically 
like the privacy 
loss of the Gaussian mechanism. As we study the asymptotic behaviour of the subsampled Gaussian 
mechanism after many 
compositions, it is possible that these theorems could be used to prove our result. 
However, we opted for another route in our proof. Instead of showing that privacy accounting 
for the subsampled Gaussian mechanism is asymptotically similar to accounting for the 
Gaussian mechanism, we show that the subsampled mechanism itself is asymptotically similar to 
the Gaussian mechanism.

\section{Background}\label{sec:background}
In this section, we go through some background material on differential privacy.
We start with the definition and basics in Section~\ref{sec:dp-basics},
look at composition in Section~\ref{sec:dp-composition} and finally 
introduce DP-SGD in Section~\ref{sec:dp-sgd}.

\subsection{Differential privacy}\label{sec:dp-basics}
Differential privacy 
(DP)~\citep{dworkCalibratingNoiseSensitivity2006,dworkAlgorithmicFoundationsDifferential2014} 
is a property of an algorithm that quantifies 
the privacy loss resulting from releasing the output of the algorithm.
In the specific definition we use, the privacy loss is quantified by 
two numbers: $\epsilon \geq 0$ and $\delta \in [0, 1]$.
\begin{definition}
    Let $\calm$ be a randomised algorithm. $\calm$ is $(\epsilon, \delta)$-DP 
    if, for all measurable $A$ and all neighbouring $x, x'\in \calx$,
    \begin{equation}
        \Pr(\calm(x) \in A) \leq e^\epsilon \Pr(\calm(x') \in A) + \delta.
    \end{equation}
\end{definition}
DP algorithms are also called \emph{mechanisms}. We use $x\sim x'$ to denote 
$x$ and $x'$ being neighbouring. The meaning of neighbouring
is domain specific. The most common definitions are \emph{add/remove}
neighbourhood, where $x\sim x'$ if they differ in only 
adding or removing one row, and \emph{substitute} neighbourhood, where 
$x\sim x'$ if they differ in at most one row, and have 
the same number of rows. We focus on the add/remove neighbourhood in this 
work.

Differential privacy is immune to post-processing, so applying any 
function to the output of a DP mechanism cannot weaken the privacy 
guarantee~\citep{dworkAlgorithmicFoundationsDifferential2014}.
\begin{theorem}\label{thm:dp-post-processing-immunity}
    Let $\calm$ be an $(\epsilon, \delta)$-DP mechanism and let $f$ be 
    any randomised algorithm. Then $f\circ \calm$ is $(\epsilon, \delta)$-DP.
\end{theorem}

The basic mechanism we look at is the Gaussian mechanism, which 
adds Gaussian noise to the output of a function~\citep{dworkOurDataOurselves2006}:
\begin{definition}\label{def:gaussian-mech}
    The Gaussian mechanism $G$ for a function $f\colon \calx \to \R^d$ 
    and noise variance $\sigma^2$
    is $G(x) = f(x) + \eta$, where $\eta \sim \caln_d(0, \sigma^2I_d)$.
\end{definition}
The privacy bounds of the Gaussian mechanism depend on the sensitivity
of $f$.
\begin{definition}\label{def:sensitivity}
    The sensitivity of a function $f$ is $\Delta = \sup_{x\sim x'}||f(x) - f(x')||_2$.
\end{definition}
Tight privacy bounds for the Gaussian mechanism were derived by 
\citet{balleImprovingGaussianMechanism2018}. We omit the exact expression here,
as it is not important for this work.

\subsection{Composition of Differential Privacy}\label{sec:dp-composition}
Another useful property of DP is composition, which means that 
running multiple DP algorithms in succession degrades the privacy 
bounds in a predictable way. The tight composition theorem we are interested 
in is most easily expressed through privacy loss random 
variables~\citep{sommerPrivacyLossClasses2019} and 
dominating pairs~\citep{zhuOptimalAccountingDifferential2022}, 
which we introduce next.

\begin{definition}\label{def:hockey-stick-divergence}
    For $\alpha \geq 0$ and random variables $P, Q$, the hockey-stick divergence is 
    \begin{equation}
        H_\alpha(P, Q) = \Exp_{t\sim Q}\left(\left(\frac{\dx P}{\dx Q}(t) - \alpha\right)_+\right),
    \end{equation}
    where $(x)_+ = \max\{x, 0\}$ and $\frac{\dx P}{\dx Q}$ is the Radon-Nikodym
    derivative, which simplifies to the density ratio if $P$ and $Q$ are 
    continuous.
\end{definition}
Differential privacy can be characterised with the hockey-stick divergence:
$\calm$ is $(\epsilon, \delta)$-DP if and only if 
$\sup_{x\sim x'} H_{e^\epsilon}(\calm(x), \calm(x')) \leq \delta$~\citep{zhuOptimalAccountingDifferential2022}.

\begin{definition}[\citealt{zhuOptimalAccountingDifferential2022}]\label{def:dominating-pair}
    A pair of random variables $P, Q$ is called a dominating pair for mechanism 
    $\calm$ if for all $\alpha \geq 0$,
    \begin{equation}
        \sup_{x\sim x'} H_\alpha(\calm(x), \calm(x')) \leq H_\alpha(P, Q).
    \end{equation}
\end{definition}
Dominating pairs allow computing privacy bounds for a mechanism from simpler 
distributions than the mechanism itself. A dominating pair for the 
Gaussian mechanism is $P = \caln(0, \sigma^2)$, 
$Q = \caln(\Delta, \sigma^2)$~\citep{KoskelaJH20}.

\begin{definition}[\citealt{sommerPrivacyLossClasses2019}]\label{def:plrv}
    For random variables $P, Q$, the privacy loss function is
    $\call(t) = \ln \frac{\dx P}{\dx Q}(t)$ and the 
    privacy loss random variable (PLRV) is $L = \call(T)$
    where $T\sim P$.
\end{definition}
The PLRV of a mechanism is formed by setting $P$ and $Q$ to be a 
dominating pair for the mechanism. For the Gaussian mechanism,
the PLRV is $\caln(\mu, 2\mu)$ with 
$\mu = \frac{\Delta^2}{2\sigma^2}$~\citep{sommerPrivacyLossClasses2019}.

The PLRV allows computing the privacy bounds of the mechanism
with a simple expectation~\citep{sommerPrivacyLossClasses2019}: 
\begin{equation}
    \delta(\epsilon) = \Exp_{s\sim L}((1 - e^{\epsilon - s})_+).\label{eq:plrv-to-delta}
\end{equation}
PLRVs also allow expressing the tight composition theorem in a simple way.
\begin{theorem}[\citealt{sommerPrivacyLossClasses2019}]\label{thm:plrv-composition}
    If $L_1, \dotsc, L_T$ are the PLRVs for mechanisms $\calm_1, \dotsc, \calm_T$,
    the PRLV of the adaptive composition of $\calm_1, \dotsc, \calm_T$ is 
    the convolution of $L_1, \dotsc, L_T$, which is the distribution of 
    $L_1 + \dotsb + L_T$.
\end{theorem}

\subsection{Differentially private SGD}\label{sec:dp-sgd}
Differentially private SGD 
(DP-SGD)~\citep{songStochasticGradientDescent2013,bassilyPrivateEmpiricalRisk2014,abadiDeepLearningDifferential2016} 
is one of the most important DP algorithms. Making SGD private 
requires \emph{clipping} the gradients to bound their sensitivity to a threshold $C$:
$\clip_C(x) = \frac{x}{||x||_2}\min\{||x||_2, C\}$, $\clip_C(0) = 0$. Then noise is 
added to the sum of the clipped gradients with the Gaussian mechanism, 
so DP-SGD uses 
\begin{equation}
    \label{eq:dp_sgd_gradients}
    G_{DP} = \sum_{i\in \calb}\clip_C(g_i) + \caln(0, \sigma^2I_d)
\end{equation}
in place of the non-private sum of gradients. $G_{DP}$ can also be used in
adaptive versions of SGD like Adam. 

To compute the privacy bounds for DP-SGD, we need to account for 
the \emph{subsampling amplification} that comes from the subsampling
in SGD. This requires fixing the subsampling scheme. We consider 
Poisson subsampling, where each datapoint in $\calb$ is included 
in the subsample with probability $q$, independently of any other 
datapoints.

When the neighbourhood relation is add/remove, the Poisson 
subsampled Gaussian mechanism, of which DP-SGD is an instance of, 
has the dominating pair~\citep{KoskelaJH20}
\begin{equation}
    \label{eq:dominating-pair}
    \begin{split}
        P &= q\caln(\Delta, \sigma^2) + (1 - q)\caln(0, \sigma^2), \\
        Q &= \caln(0, \sigma^2).
    \end{split}
\end{equation}
With this dominating pair, we can form the PLRV for the mechanism, take 
a $T$-fold convolution for $T$ iterations, and
compute the privacy bounds from the expectation in \eqref{eq:plrv-to-delta}. 
The computation is not 
trivial, but can be done using numerical privacy 
accountants~\citep{KoskelaJH20,koskelaTightDifferentialPrivacy2021,GopiLW21,doroshenkoConnectDotsTighter2022a,alghamdiSaddlePointMethodDifferential2023}.
These accountants are used by libraries implementing DP-SGD like 
Opacus~\citep{yousefpourOpacusUserFriendlyDifferential2021a}
to compute privacy bounds in practice.

\subsection{Kullback-Leibler Divergence and Total Variation Distance}

Our proofs use two notions of distance between random variables, \emph{total variation distance} 
and \emph{Kullback-Leibler (KL) divergence}~\citep{kullbackInformationSufficiency1951}.
\begin{definition}
    Let $P$ and $Q$ be random variables.
    \begin{enumerate}
        \item 
        The total variation distance between $P$ and $Q$ is 
        \begin{equation}
            \TV(P, Q) = \sup_A |\Pr(P\in A) - \Pr(Q \in A)|.
        \end{equation}
        The supremum is over all measurable sets $A$.
        \item 
        The KL divergence between $P$ and $Q$ with densities $p(t)$ and $q(t)$ is 
        \begin{equation}
            \KL(P, Q) = \Exp_{t\sim P}\left(\ln \frac{p(t)}{q(t)}\right).
        \end{equation}
    \end{enumerate}
\end{definition}
The two notions of distance are related by Pinsker's inequality.
\begin{lemma}[\citealp{kelbertSurveyDistancesMost2023}]
    For random variables $P, Q$,
    \begin{equation}
        \TV(P, Q) \leq \sqrt{\frac{1}{2}\KL(P, Q)}.
    \end{equation}
\end{lemma}

\section{Accounting Oracles}\label{sec:accounting-oracles}

We use the concept of \emph{accounting oracles}~\citep{tajeddinePrivacypreservingDataSharing2020}
to make formalising properties of privacy accounting easier. The accounting oracle 
is the ideal privacy accountant that numerical accountants aim to approximate.

\begin{definition}
    For a mechanism $\calm$, the accounting oracle $\AO_\calm(\epsilon)$
    returns the smallest $\delta$ such that $\calm$ is $(\epsilon, \delta)$-DP.
\end{definition}
In case $\calm$ has hyperparameters that affect its privacy bounds, these 
hyperparameters will also be arguments of $\AO_\calm$.
We write the accounting oracle for the $T$-fold composition of the Poisson subsampled 
Gaussian mechanism as 
$\AO_S(\sigma, \Delta, q, T, \epsilon)$, and the accounting oracle of a composition of 
the Gaussian mechanism as $\AO_G(\sigma, \Delta, T, \epsilon)$.

Accounting oracles make it easy to formally express symmetries of privacy accounting. 
For example, privacy bounds are invariant to post-processing with a bijection.
\begin{lemma}\label{lemma:bijective-post-processing-immunity}
    Let $\calm$ be a mechanism and $f$ be a bijection. Then
        $\AO_\calm(\epsilon) = \AO_{f\circ \calm}(\epsilon)$.
\end{lemma}
\begin{proof}
    This follows by using post-processing immunity for $f$ to show that 
    $\AO_{f \circ \calm}(\epsilon) \leq \AO_\calm(\epsilon)$ and for 
    $f^{-1}$ to show that $\AO_{f \circ \calm}(\epsilon) \geq \AO_\calm(\epsilon)$.
\end{proof}
We can also formalise the lemma that considering 
$\Delta = 1$ is sufficient when analysing the (subsampled) Gaussian mechanism.
\begin{lemma}\label{lemma:sensitivity-to-noise-exchange}
    \begin{align}
        \AO_G(\sigma, \Delta, T, \epsilon) &= \AO_G(\sigma/\Delta, 1, T, \epsilon), \\
        \AO_S(\sigma, \Delta, q, T, \epsilon) &= \AO_S(\sigma/\Delta, 1, q, T, \epsilon).
    \end{align}
\end{lemma}
\begin{proof}
    Let $\calm$ be the (subsampled) Gaussian mechanism. Then 
    \begin{equation}
        \calm(x) = f(x) + \eta = \Delta\left(\frac{1}{\Delta} f(x) + \frac{1}{\Delta}\eta \right)
    \end{equation}
    with $\eta \sim \caln_d(0, \sigma^2 I_d)$.
    The sensitivity of $\frac{1}{\Delta}f(x)$ is 1, so the part inside parenthesis
    in the last expression is a (subsampled) Gaussian mechanism with sensitivity 1
    and noise standard deviation $\sigma/\Delta$. Multiplying by $\Delta$ is bijective
    post-processing, so that mechanism must have the same privacy bounds 
    as the original mechanism. In a composition, this transformation can be done separately 
    for each individual mechanism of the composition.
\end{proof}
Since considering $\Delta = 1$ when analysing the subsampled Gaussian mechanism is enough 
by Lemma~\ref{lemma:sensitivity-to-noise-exchange}, we occasionally shorten 
$\AO_S(\sigma, 1, q, T, \epsilon)$ to $\AO_S(\sigma, q, T, \epsilon)$.

The next lemma and corollary show that mechanisms close in total variation 
distance also have similar privacy bounds.
\begin{lemma}\label{lemma:dp-total-variation-distance-bound}
    Let $\calm$ be an $(\epsilon, \delta)$-DP mechanism, and let 
    $\calm'$ be a mechanism with 
    \begin{equation}
        \sup_x \TV(\calm(x), \calm'(x)) \leq d,
    \end{equation}
    for some $d \geq 0$.
    Then $\calm'$ is $(\epsilon, \delta + (1 + e^\epsilon)d)$-DP.
\end{lemma}
\begin{proof}
    For any measurable set $A$,
    \begin{equation}
        \begin{split}
            \Pr(\calm'(x) \in A) 
            &\leq \Pr(\calm(x) \in A) + d
            \\&\leq e^\epsilon\Pr(\calm(x') \in A) + \delta + d
            \\&\leq e^\epsilon(\Pr(\calm'(x') \in A) + d) + \delta + d
            \\&= e^\epsilon\Pr(\calm'(x') \in A) + \delta 
            + (1 + e^\epsilon)d.
        \end{split}
    \end{equation}
\end{proof}
Lemma~\ref{lemma:dp-total-variation-distance-bound} can also be expressed with 
accounting oracles.
\begin{corollary}\label{corollary:ao-total-variation-distance-bound}
    Let $\AO_\calm$ and $\AO_{\calm'}$ be the accounting oracles for mechanisms
    $\calm$ and $\calm'$, respectively. If 
    \begin{equation}
        \sup_x \TV(\calm(x), \calm'(x)) \leq d.
    \end{equation}
    then
    \begin{equation}
        |\AO_\calm(\epsilon) - \AO_{\calm'}(\epsilon)| \leq (1 + e^\epsilon)d.
    \end{equation}
\end{corollary}
\begin{proof}
    Let $\delta = \AO_\calm(\epsilon)$. By Lemma~\ref{lemma:dp-total-variation-distance-bound},
    $\calm'$ is $(\epsilon, (1 + e^\epsilon)d + \delta)$-DP, so
    \begin{equation}
        \AO_{\calm'}(\epsilon) \leq \delta + (1 + e^\epsilon)d. 
        \label{eq:corollary-ao-total-variation-distance-bound-1}
    \end{equation}
    If 
    $\AO_{\calm'}(\epsilon) < \delta - (1 + e^\epsilon)d$, 
    \begin{equation}
        \AO_\calm(\epsilon) < \delta - (1 + e^\epsilon)d + (1 + e^\epsilon)d = \delta
    \end{equation}
    by Lemma~\ref{lemma:dp-total-variation-distance-bound}, which is a contradiction,
    so 
    \begin{equation}
        \AO_{\calm'}(\epsilon) \geq \delta - (1 + e^\epsilon)d. 
        \label{eq:corollary-ao-total-variation-distance-bound-2}
    \end{equation}
    Combining \eqref{eq:corollary-ao-total-variation-distance-bound-1} and
    \eqref{eq:corollary-ao-total-variation-distance-bound-2},
    \begin{equation}
        |\AO_\calm(\epsilon) - \AO_{\calm'}(\epsilon)| \leq (1 + e^\epsilon)d. \qedhere
    \end{equation}
\end{proof}

\section{DP-SGD noise decomposition}
\label{sec:noise-decomposition}
For now, let us denote the sum on the right in Equation \eqref{eq:dp_sgd_gradients},
the sum of clipped gradients with subsampling rate $q$, with $G_q$:
\begin{align}
    G_q = 
        \sum_{i\in \calb}\clip_C(g_i).
\end{align}
In each step of DP-SGD, we are releasing this sum using Gaussian perturbation.
However, due to the subsampling, our gradient sum can comprise of any number of terms between
$0$ and $N$. Therefore, before we do the step, we want to upscale the summed
gradient to approximate the magnitude of the full data gradient $G_1$. This is
typically done by scaling the $G_q$ with $N/|B|$ where $|B|$ denotes the size of
the minibatch. However, using the $|B|$ as such might leak private information.
Instead, we will use $1/q$ scaling for $G_q$, which also gives an unbiased
estimator of $G_1$:
\begin{align}
    \Exp \left[ \frac{1}{q}G_q - G_1 \right] 
        &= \Exp \left[ \frac{1}{q}\sum_{i \in [N]}b_i \tilde{g_i} - \sum_{i \in [N]} \tilde{g_i} \right] \\
        & = \sum_{i \in [N]} \tilde{g_i} \Exp \left[ \frac{b_i}{q} - 1 \right] = 0,
\end{align}
where $b_i \sim \mathrm{Bernoulli}(q)$ and $\tilde{g_i}$ denote the clipped per-example gradients.
Using Lemma \ref{lemma:sensitivity-to-noise-exchange}, we can decouple the
noise-scale and the clipping threshold $C$ and write the DP gradient $\dpG$ used to
update the parameters as
\begin{align}
    \label{eq:unbiased-grad}
    \dpG = \frac{1}{q} \left( G_q + C \sigma \eta\right),
\end{align}
where $\eta \sim \caln(0, I_d)$. The clipping threshold $C$ affects
the update as a constant scale independent of $q$, and therefore
for notational simplicity we set $C=1$. Since the subsampled 
gradient $G_q$ and the DP noise are independent, we can decompose the total gradient variance 
of the $j$th element of $\dpG$ as 
\begin{align}
    \label{eq:variance-decomp}
    \underbrace{\Var(\dpG_j)}_{\text{Total}} 
    = \underbrace{\frac{1}{q^2} \Var(G_{q, j})}_{\text{Subsampling}} 
    + \underbrace{\lp\frac{\sigma}{q}\rp^2}_{\text{Effective DP Noise}}.
\end{align}
The first component on the right in Equation \eqref{eq:variance-decomp}, the
subsampling variance, arises from the subsampling, so it is easy to see, as we
show in \cref{app:subsampling-variance}, that it is decreasing w.r.t.\ $q$,
which holds regardless of the clipping threshold assuming $C > 0$. For the rest
of the paper, we will use $\sigma_{\text{eff}}^2 := \sigma^2 / q^2$ to denote
the second component of the variance decomposition in \cref{eq:variance-decomp},
the effective noise variance.

In order to guarantee $(\epsilon, \delta)$-DP, the $\sigma$ term in $\sigma_{\text{eff}}$ needs to 
be adjusted based on the number of iterations $T$ and the subsampling rate 
$q$. Hence we will treat the $\sigma$ as a function of $q$ and $T$, and denote 
the smallest $\sigma$ that provides $(\epsilon, \delta)$-DP as $\sigma(q, T)$. 

Now, the interesting question is, at what rate does the $\sigma(q, T)$ grow w.r.t.~$q$? 
In \cref{eq:variance-decomp}, we saw that the effective variance arising from
the DP noise addition scales as $(\sigma(q,T) / q)^2$. 
In \cref{sec:asym-linear-sigma}, we show that as $T \rightarrow \infty$, 
$\sigma(q, T)$ becomes a linear function w.r.t.\ $q$.
Since $\sigma(1, T) = 1\cdot \sigma(1, T)$, this means that
\begin{equation}
    \sigma(q, T) \approx q\sigma(1, T). 
\end{equation}
This result implies that subsampling does not reduce the effective level of 
injected DP noise, while we would have more subsampling-induced noise arising from the first term
of Equation \eqref{eq:variance-decomp} for small values of $q$. 

In Section~\ref{sec:empirical-results} we study $\sigma_{\text{eff}}$ 
numerically, and find that as $T$ grows, the $\sigma_{\text{eff}}$ decreases 
monotonically towards $\sigma(1, T)$.

Furthermore, in \cref{sec:no-comp} we show that when $T=1$, the subsampling 
can in fact be harmful and smaller values $q$ incur a proportionally larger effective noise 
variance in the gradient updates.

\section{Subsampled Gaussian Mechanism in the Limit of Many Compositions}
\label{sec:asym-linear-sigma}

Since privacy accounting for DP-SGD is based on the accounting oracle $\AO_S(\sigma, \Delta, q, T, \epsilon)$ for the Poisson
subsampled Gaussian mechanism, it suffices to study the behaviour of $\AO_S$, instead of DP-SGD directly.
This can be done in a much simpler setting. 

In the setting we use, there is a single datapoint $x\in \{0, 1\}$,
which is released $T$ times with the Poisson subsampled Gaussian mechanism:
\begin{equation}
    \calm_i(x) \sim xB_q + \caln(0, \sigma^2_T)
\end{equation}
for $1\leq t \leq T$, where $B_q$ is a Bernoulli random variable.
Since $\calm$ is a Poisson subsampled Gaussian mechanism, its accounting
oracle is $\AO_S(\sigma, 1, q, T, \epsilon)$.\footnote{
Considering $\Delta = 1$ suffices due to Lemma~\ref{lemma:sensitivity-to-noise-exchange}.
}
We have 
\begin{align}
    \Exp(\calm_i(x)) &= qx \\
    \Var(\calm_i(x)) &= x^2q(1 - q) + \sigma_T^2.
\end{align}
As $\sigma_T^2\to \infty$, the variance and $\sigma_T^2$ approach each other.
As a result, we can approximate $\calm_i$ with 
\begin{equation}
    \calm_i'(x) \sim \caln(qx, \sigma_T^2).
\end{equation}

To prove this, we first need to find a lower bound on how quickly 
$\sigma_T^2$ must grow as $T\to \infty$.
\begin{restatable}{theorem}{theoremsigmagrowth}\label{thm:sigma-growth}
    Let $\sigma_T$ be such that $\AO_S(\sigma_T, \Delta, q, T, \epsilon) \leq \delta$ for all 
    $T$, with $\delta < 1$. Then $\sigma_T^2 = \Omega(T)$.
\end{restatable}
\begin{proof}
    The idea of the proof is to show that unless $\sigma_T^2 = \Omega(T)$,
    the mean of the PLRV can obtain arbitrarily large values, and then
    show that this leads to violating the $\delta$ bound.
    We defer the full proof to Appendix~\ref{app:sigma-growth}.
\end{proof}

Now we can prove that the approximation is sound.
\begin{restatable}{theorem}{theoremsubsampledmechconvergence}\label{thm:subsampled-mech-convergence}
    For $1 \leq i \leq T$, let 
    \begin{align}
        \calm_i(x) &\sim xB_q + \caln(0, \sigma_T^2), \label{eq:simple-setting-mog}\\
        \calm_i'(x) &\sim \caln(qx, \sigma_T^2).
    \end{align}
    be independent for each $i$. Let $\calm_{1:T}$ be the composition
    of $\calm_i$, and let $\calm_{1:T}'$ be the composition of $\calm'_i$.
    Then
    \begin{equation}
        \sup_x \TV(\calm_{1:T}(x), \calm_{1:T}'(x)) \to 0
    \end{equation}
    as $T\to \infty$.
\end{restatable}
\begin{proof}
    It suffices to show
    \begin{equation}
        \sup_x T\cdot \KL(\calm_i'(x), \calm_i(x)) \to 0
    \end{equation}
    due to Pinsker's inequality and the additivity of KL divergence for products of independent
    random variables. When $x = 0$, the two mechanism are the same, so it suffices to 
    look at $x = 1$. The idea for the rest of the proof is to first reparameterise 
    $\sigma$ in terms of a variable $u$ with $u \to 0$ as $T\to \infty$ and then 
    use Taylor's theorem at $u = 0$ to find the convergence rate of the KL divergence.
    For the complete proof see Appendix~\ref{app:subsampled-mech-convergence}.
\end{proof}

The mechanism $\calm_i'$ is nearly a Gaussian mechanism, since 
\begin{equation}
    \frac{1}{q}\calm_i'(x) \sim x + \caln\left(0, \frac{\sigma_T^2}{q^2}\right).\label{eq:approx-mech-is-almost-gaussian}
\end{equation}
On the right is a Gaussian mechanism with noise standard deviation $\frac{\sigma_T}{q}$, which has
the linear relationship between $\sigma_T$ and $q$ we are looking for. In the next theorem, we use this 
to prove our main result.
\begin{theorem}\label{thm:subsampled-mech-ao-convergence}
    For any $\sigma$, $q_1$, $q_2$, $\Delta$ and $\epsilon$
    \begin{equation}
        |\AO_S(\sigma, \Delta, q_1, T, \epsilon) - \AO_S(\sigma\cdot q_2 / q_1, \Delta, q_2, T, \epsilon)| \to 0
    \end{equation}
    as $T \to \infty$.
\end{theorem}
\begin{proof}
    It suffices to look at $\Delta = 1$ by Lemma~\ref{lemma:sensitivity-to-noise-exchange}.
    Let $\AO'(\sigma, q, T, \epsilon)$ be the accounting oracle for $\calm_{1:T}'$.
    By Lemma~\ref{lemma:bijective-post-processing-immunity} and \eqref{eq:approx-mech-is-almost-gaussian},
    \begin{equation}
        \AO'(\sigma, q, T, \epsilon) = \AO_G\left(\frac{\sigma}{q}, T, \epsilon\right).\label{eq:approx-mech-is-almost-gaussian-ao}
    \end{equation}

    Let $\sigma_2 = \sigma \cdot q_2 / q_1$,
    \begin{equation}
        d_T^{(1)} = \sup_x\TV(\calm_{1:T}(x, q_1, \sigma), \calm'_{1:T}(x, q_1, \sigma)),
    \end{equation}
    and 
    \begin{equation}
        d_T^{(2)} = \sup_x\TV(\calm_{1:T}(x, q_2, \sigma_2), \calm'_{1:T}(x, q_2, \sigma_2)).
    \end{equation}
    Now
    \begin{equation}
        |\AO_S(\sigma, q_1, T, \epsilon) - \AO'(\sigma, q_1, T, \epsilon)|
        \leq (1 + e^\epsilon)d_T^{(1)}
    \end{equation}
    and
    \begin{equation}
        |\AO_S(\sigma_2, q_2, T, \epsilon) - \AO'(\sigma_2, q_2, T, \epsilon)|
        \leq (1 + e^\epsilon)d_T^{(2)}
    \end{equation}
    by Corollary~\ref{corollary:ao-total-variation-distance-bound}.
    Since $\frac{\sigma}{q_1} = \frac{\sigma_2}{q_2}$,
    \begin{equation}
        \begin{split}
            \AO'(\sigma, q_1, T, \epsilon) 
            &= \AO_G\left(\frac{\sigma}{q_1}, T, \epsilon\right)
            \\&= \AO_G\left(\frac{\sigma_2}{q_2}, T, \epsilon\right)
            \\&= \AO'(\sigma_2, q_2, T, \epsilon).
        \end{split}
    \end{equation}
    Now 
    \begin{equation}
        \begin{split}
            &|\AO_S(\sigma, q_1, T, \epsilon) - \AO_S(\sigma\cdot q_2 / q_1, q_2, T, \epsilon)| 
            \\&\leq|\AO_S(\sigma, q_1, T, \epsilon) - \AO'(\sigma, q_1, T, \epsilon)| 
            \\&\phantom{x}+|\AO_S(\sigma_2, q_2, T, \epsilon) - \AO'(\sigma_2, q_2, T, \epsilon)| 
            \\&\leq (1 + e^\epsilon)(d_T^{(1)} + d_T^{(2)})
            \to 0
        \end{split}
    \end{equation}
    by Theorem~\ref{thm:subsampled-mech-convergence}.
\end{proof}

\begin{figure*}[t]
    \centering
    \includegraphics{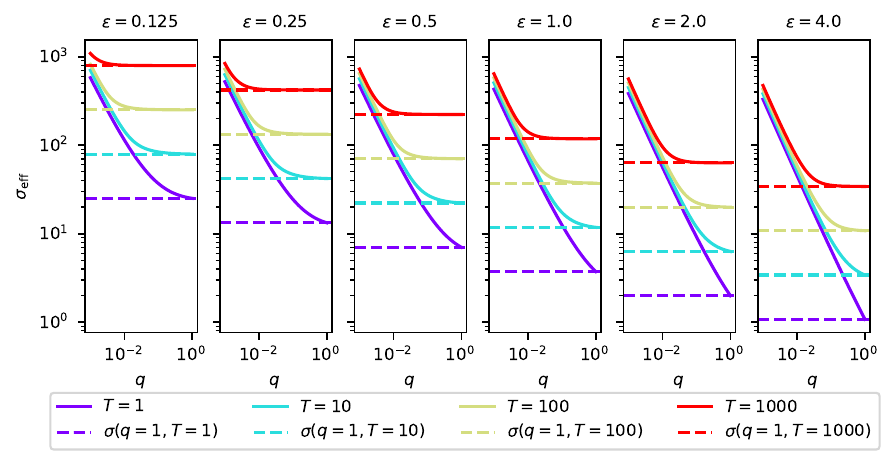}
    \vspace{-5mm}
    \caption{
        The $\seff:=\sigma(q, T)/q$ decreases as $q$ grows for all the $\epsilon$ and $T$ values. 
        As $T$ grows, $\seff$ approaches the $\sigma(1, T)$, as the asymptotic theory predicts.
        The privacy parameter $\delta$ was set to $10^{-5}$ when computing the $\sigma(q, T)$.
        }
    \label{fig:convergence-plot}
\end{figure*}

By setting $q_2 = 1$, we see that $\AO_S(\sigma, \Delta, q, T, \epsilon)$ behaves 
like $\AO_S(\sigma/q, \Delta, 1, T, \epsilon)$ in the $T\to \infty$ limit, so
$\frac{\sigma}{q}$ must be a constant that does not depend on $q$ to reach a 
given $(\epsilon, \delta)$-bound. We formalise this in the following corollary.
\begin{restatable}{corollary}{corollarylinearsigmalimit}\label{corollary:linear-sigma-limit}
    Let $\sigma(q, T)$ be the smallest $\sigma$ such that 
    $\AO_S(\sigma, \Delta, q, T, \epsilon) \leq \delta$. Then 
    \begin{equation}
        \lim_{T\to \infty} \frac{\sigma(q, T)}{q\sigma(1, T)} = 1.
    \end{equation}
\end{restatable}
\begin{proof}
    The claim follows from Theorem~\ref{thm:subsampled-mech-ao-convergence} and the 
    continuity of the inverse of 
    $\AO_S(\sigma, \Delta, 1, T, \epsilon) = \AO_G(\sigma, \Delta, T, \epsilon)$
    when considered only a function of $\sigma$, with other parameters fixed. 
    We defer the full proof to Appendix~\ref{app:linear-sigma-limit}.
\end{proof}


\subsection{Empirical results}\label{sec:empirical-results}

In order to demonstrate how the $\sigma_{\text{eff}}$ converges to $q \sigma(1,
T)$, we use Google's open source implementation\footnote{Available at
\url{https://github.com/google/differential-privacy/tree/main/python}} of the PLD
accountant \cite{doroshenkoConnectDotsTighter2022a} to compute the $\sigma(q,
T)$.

\cref{fig:convergence-plot} shows that as the the number of iterations grows,
the $\seff$ approaches the $\sigma(1, T)$ line. We can also see that for
smaller values of $\epsilon$, the convergence happens faster. This behaviour can be
explained by the larger level of noise needed for smaller $\epsilon$
values, which will make the components of the Gaussian mixture in
\cref{eq:simple-setting-mog} less distinguishable from each other.

We can also see that for all of the settings, the $\seff$ stays above the
$\sigma(1, T)$ line, and it is the furthest away when $q$ is the smallest. This
would suggest, that our observation in \cref{sec:no-comp} hold also for $T > 1$,
and that smaller values of $q$ incur a disproportionally large DP-induced
variance component in \cref{eq:variance-decomp}.

The source code for reproducing our experiments can be found in 
\url{https://github.com/DPBayes/subsampling-is-not-magic}.

\section{Subsampled Gaussian Mechanism without Compositions}
\label{sec:no-comp}
For now, let us consider $T=1$, and denote $\sigma(q) := \sigma(q, 1)$. 
We can express the $\delta$ using the hockey-stick divergence 
and the dominating pair $(P, Q)$ defined in \cref{eq:dominating-pair}
as 
\begin{equation}
    \label{eq:delta-no-comp}
    \begin{split}
        \delta(q) = &q\Pr (Z  \geq \sigma(q) \log \lp \frac{h(q)}{q} \rp - \frac{1}{2\sigma(q)}) \\
           &-h(q) \Pr(Z  \geq \sigma(q) \log \lp \frac{h(q)}{q} \rp + \frac{1}{2\sigma(q)}),
    \end{split}
\end{equation}
where $h(q) := e^\epsilon - (1-q)$. Recall that
we have assumed $\sigma(q, T)$ to be a function that returns a noise-level
matching to a particular $(\epsilon,\delta)$ privacy level for given $q$ and $T$. 
Therefore, the derivative of $\delta(q)$ in \cref{eq:delta-no-comp} w.r.t.\ 
$q$ is $0$, and we can solve the derivative of the RHS for $\sigma'(q)$.
Let us denote 
\begin{equation}
    \label{eq:def-a}
    a := \frac{1}{2\sqrt{2} \sigma(q)}
\end{equation}
\begin{equation}
    \label{eq:def-b}
    b := \frac{\sigma(q)}{\sqrt{2}} \log \lp \frac{e^\epsilon - (1-q)}{q} \rp.
\end{equation}
We have the following Lemma
\begin{restatable}{lemma}{lemmanocompderivative}
    \label{lemma:no-comp-derivative}
    For the smallest $\sigma(q)$ that provides $(\epsilon, \delta)$-DP for the Poisson subsampled
    Gaussian mechanism, we have
    \begin{align}
        \label{eq:sigma-no-comp-der}
        \sigma'(q) 
            &= \frac{\sigma(q)}{q} \frac{1}{2 a} \frac{1}{\erf'(a-b)} (\erf(a-b) - \erf(-a-b)).
    \end{align}
\end{restatable}
\begin{proof}    
    See \cref{sec:app-sigma-derivative}.
\end{proof}
Now, \cref{lemma:no-comp-derivative} allows us to establish following result.
\begin{theorem}
    \label{thm:no-comp-result}
    If $a < b$ for $a$ and $b$ defined in Equations~\eqref{eq:def-a} and \eqref{eq:def-b}, then
        $\frac{\dd}{\dd q} \frac{\sigma(q)}{q} < 0$.
\end{theorem}
\begin{proof}
   The $\erf(x)$ is a convex function for $x \in \mathbb{R}_{<0}$. Since $a, b \geq 0$
   we have $-a-b \leq 0$ and if $a-b < 0$ we get from the convexity that
   \begin{align}
       \erf(a-b) - \erf(-a-b) 
            &< 2a \erf'(a-b).
   \end{align}
   Substituting this upper bound into \cref{eq:sigma-no-comp-der} gives
   \begin{align}
       \sigma'(q) < \frac{\sigma(q)}{q} \Leftrightarrow \frac{q\sigma'(q)-\sigma(q)}{q^2} = \frac{\dd}{\dd q} \frac{\sigma(q)}{q} < 0.
   \end{align}
\end{proof}
Now, \cref{thm:no-comp-result} implies that $\sigma_{\text{eff}}$ is a decreasing 
function w.r.t.\ $q$, and therefore larger subsampling rates should be preferred when
$a < b$. So now the remaining question is, when is $a$ smaller than $b$. It is 
easy to see that if $a > b$, we have an upper bound on the $\sigma(q)$, and therefore
we cannot obtain arbitrarily strict privacy levels. 

However, analytically solving the region where $a < b$ is intractable as we do
not have a closed form expression for $\sigma(q)$. Therefore we make the following
Conjecture, which we study numerically.
\begin{conjecture}
    \label{conj:a-minus-b}
    For $\epsilon, q \geq 4 \delta$, we have $a-b < 0$.
\end{conjecture}

Verifying the \cref{conj:a-minus-b} numerically requires computing the $\sigma(q)$ values
for a range of $q$ and $\epsilon$ values, which would be computationally inefficient.
However, computing $a, b$ and $\delta$ can be easily parallelized over multiple
values of $q$ and $\sigma$. We set $\delta_{\text{target}}=10^{-5}$ and compute the 
$a, b$ and $\delta$ for $q \in [4 \delta_{\text{target}}, 1.0]$, 
$\sigma \in [\min(q) \sigma(1, \epsilon), \sigma(1, \epsilon)]$ and $\epsilon \in [4\delta_{\text{target}}, 4.0]$ 
which would be a reasonable range of $\epsilon$ values for practical use. The 
$\sigma(1, \epsilon)$ is a noise-level matching $(\epsilon, \delta_{\text{target}})$-DP guarantee for $q=1$,
and the lower bound $\min(q) \sigma(1, \epsilon)$ for $\sigma$ was selected based on hypothesis 
that within $[\min(q) \sigma(1, \epsilon), \sigma(1, \epsilon)]$ we can find a noise-level closely
matching the $\delta_{\text{target}}$.

\cref{fig:a-minus-b} shows the largest $a-b$ value for our target $\epsilon$
values among the $q$ and $\sigma$ pairs that resulted into $\delta \leq \delta_{\text{target}}$.
We can see that among these values there are no cases of $a > b$.
While our evaluation covers a range of $\epsilon$ values,
$a-b$ seems to be monotonically decreasing w.r.t.\ $\epsilon$, which 
suggests that the conjecture should hold even for larger values of $\epsilon$.
Based on our numerical evaluation, the $\sigma$ range
resulted into $\delta$ values differing from $\delta_{\text{target}}=10^{-5}$ 
at most $\approx 2\times10^{-7}$ in absolute difference.
As a final remark, the constant $4$ in the \cref{conj:a-minus-b}
was found numerically. With smaller values of this constant we obtained
cases for which $a - b > 0$. Furthermore, values of $q \approx \delta$ empirically produce results that even fail to satisfy the claim of \cref{thm:no-comp-result}.

\begin{figure}
    \centering
    \includegraphics{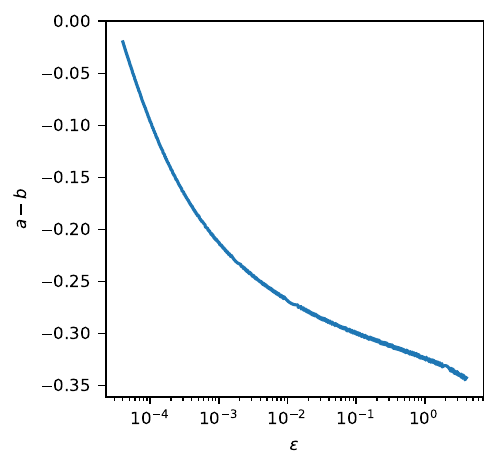}
    \caption{
        The largest $a-b$ computed for multiple $(q,\sigma)$ pairs stays negative
        for a broad range of $\epsilon$ values. The $a$ and $b$ were selected so
        that the corresponding $(q, \sigma)$ pair satisfies $(\epsilon, \delta)$-DP 
        with $\delta \leq 10^{-5}$.
    }
    \label{fig:a-minus-b}
\end{figure}


\paragraph{Limitations}
Our analysis for the $T=1$ case reduced the monotonicity of effective DP noise standard deviation $\seff$ to a sufficient condition $a < b$.
We were unable to provide a formal proof of when this condition is satisfied, but verified numerically that it appears valid for a broad range of practically relevant parameters with $\delta=10^{-5}$, a standard value suggested by NIST\footnote{\url{https://doi.org/10.6028/NIST.SP.800-226.ipd}}.

\section{Discussion}
In this paper, we have studied how the subsampling rate affects the level of Gaussian 
perturbation in Poisson subsampled DP-SGD. An important question is how this affects
the convergence of DP-SGD. Recently, \citet{Bu0ZK23} studied the convergence of DP-SGD using 
gradient normalisation instead of clipping.
In their work, they proved a convergence result for DP-SGD with expected batch size $B$, 
number of iterations $T$ and Gaussian noise with standard deviation of $\sigma$. Their Theorem 
4~gives an upper bound on the expected gradient norm as 
\begin{align*}
    \min_{0 \leq t \leq T} \mathbb{E}(||g_t||) 
        \leq \mathcal{G}\left( \frac{4}{\sqrt{T}} \sqrt{(\mathcal{L}_0-\mathcal{L}_{*} )L\left(1 + \frac{\sigma^2 d}{B^2}\right)}; \xi, \gamma \right)
\end{align*}
where the $\mathcal{L}_0, \mathcal{L}*$ and $L$ are regularity condition constants for the loss 
function, $\gamma$ is a stability parameter for the gradient normalisation, and 
$\xi^2$ is the sum of the total gradient variances for each dimension from 
\eqref{eq:variance-decomp}. $\mathcal{G}$ is a complicated function that arises from the gradient 
normalisation.

The bound is affected by the noise-level and the batch size in terms of the factor 
$\frac{\sigma^2}{B^2}$. This term is linearly proportional to the $\frac{\sigma(q)^2}{q^2}$ which 
we study in this paper. Thus, given the same number of iterations, our results that 
$\frac{\sigma(q)}{q}$ is non-increasing w.r.t. $q$ would also imply that the 
$\frac{\sigma^2}{B^2}$ term in the above bound is also non-increasing w.r.t $q$. 
The function $\mathcal{G}$ in this Theorem is increasing with regards to the first parameter, 
which is the only place the number of iterations $T$ appears.
This means that the batch size does not affect the convergence rate in the asymptotic regime,
and a large batch size improves the constant for the convergence rate in the 
non-asymptotic regime.

The effect of the batch size through $\xi$ is harder to reason about due to the 
complicated nature of $\mathcal{G}$.
$\mathcal{G}$ appears to be increasing with regards to $\xi$ based on numerical results 
of \citet{Bu0ZK23}, which would mean that a larger batch 
size also improves the convergence through reducing $\xi$.
When $\gamma = 0$, \citet{Bu0ZK23} show that $\mathcal{G} \geq \xi$, 
so a larger batch size improves the quality of the solution DP-SGD can find in this case.

\section{Conclusion}
We studied the relationship between the effective noise variance and the subsampling
rate in Poisson subsampled DP-SGD, and proved that as the number of iterations
approaches infinity, the relationship becomes linear, which cancels 
the effect of the subsampling rate in the effective noise variance. This means 
that a large subsampling rate always reduces the effective total gradient variance.
Furthermore, we demonstrated that under a wide range of $\epsilon$ values,
a single application of the Poisson subsampled Gaussian mechanism actually
incurs a monotonically decreasing effective noise variance w.r.t.\  subsampling rate.
Our numerical experiments show that the asymptotic regime is relevant in practice, and
that outside the asymptotic regime, smaller subsampling rates lead to increasingly large effective total gradient variances. This explains the observed benefits of large batch sizes 
in DP-SGD, which has so far had only empirical and heuristic explanations, furthering the 
theoretical understanding of DP-SGD.

For future work, it would be important to theoretically study how to interpolate
our results between the $T=1$ and the asymptotic case. Based on our numerical
evaluations however, we expect our main conclusion, that the large batches
provide smaller effective noise, to hold even for finite $T>1$.

In all cases we have studied, less subsampling (larger $q$) always leads to better privacy--utility trade-off, at the cost of more compute.
Thus the magic of subsampling amplification lies in saving compute, not in achieving higher utility than without subsampling.

\section*{Acknowledgements}
This work was supported by the Research Council of Finland 
(Flagship programme: Finnish Center for Artificial Intelligence, 
FCAI as well as Grants 356499 and 359111), the Strategic Research Council 
at the Research Council of Finland (Grant 358247)
as well as the European Union (Project
101070617). Views and opinions expressed are however
those of the author(s) only and do not necessarily reflect
those of the European Union or the European Commission. Neither the European 
Union nor the granting authority can be held responsible for them.

\section*{Impact Statement}
This paper presents work whose goal is to advance the field of 
trustworthy machine learning.
Our results improve the understanding of DP-SGD and help designing 
better DP algorithms. As such, we believe work like ours has an 
overall positive effect on machine learning.

\bibliography{references.bib}
\bibliographystyle{icml2024}

\newpage
\appendix
\onecolumn

\section{Missing proofs}\label{app:missing-proofs}

\subsection{The subsampling induced variance decreases w.r.t.\ $q$}
\label{app:subsampling-variance}
Recall from \cref{sec:noise-decomposition}, that we denote the sum of clipped
subsampled gradients with $G_q$:
\begin{align}
    G_q = 
        \sum_{i\in \calb}\clip_C(g_i).
\end{align}
Now for the subsampling induced variance in noise decomposition of \cref{eq:variance-decomp}
we have
\begin{align}
    \Var\lp \frac{1}{q} G_{q, j} \rp 
        &= \frac{1}{q^2} \Var\lp \sum_{i \in [N]} b_i g_{i,j} \rp \\
        &= \frac{1}{q^2}  \sum_{i \in [N]} g_{i,j}^2 \Var\lp b_i  \rp \\
        &= \frac{q(1-q)}{q^2}  \sum_{i \in [N]} g_{i,j}^2 \\
        &= \frac{1-q}{q}  \sum_{i \in [N]} g_{i,j}^2 \label{eq:subsampling-variance-last}.
\end{align}
Now, it is easy to see that the sum in \cref{eq:subsampling-variance-last} is a 
constant w.r.t.\ $q$, and the term $(1-q)/q$ is decreasing w.r.t.\ $q$. Thus
the subsampling induced variance is decreasing w.r.t.\ $q$.

\subsection{Useful Lemmas}\label{app:lemmas}

\begin{lemma}[\citealp{kullbackInformationSufficiency1951}]\label{lemma:kl-div-properties}
    Properties of KL divergence:
    \begin{enumerate}
        \item If $P$ and $Q$ are joint distributions over independent random 
        variables $P_1, \dotsc, P_T$ and $Q_1, \dotsc, Q_T$,
        \begin{equation}
            \KL(P, Q) = \sum_{i=1}^T \KL(P_i, Q_i).
        \end{equation}
        \item If $f$ is a bijection, 
        \begin{equation}
            \KL(f(P), f(Q)) = \KL(P, Q).
        \end{equation}
    \end{enumerate}
\end{lemma}

We will need to analyse the following function:
\begin{equation}
    f_x(u) = \ln\frac{\caln(x; 0, 1)}{q\caln(x; u - qu, 1) + (1 - q)\caln(x; -qu, 1)}.
\end{equation}
In particular, we need the fourth-order Taylor approximation of 
$\Exp_x(f_x(u))$ for $x\sim \caln(0, 1)$ at $u = 0$. We begin by 
looking at the Taylor approximation of $f_x(u)$ without the expectation,
and then show that we can differentiate under the expectation.

\begin{lemma}\label{lemma:f-taylor-approximation}
    The fourth-order Taylor approximation of $f_x$ at $u = 0$ is 
    \begin{equation}
        \begin{split}
            f_x(u) &= \frac{1}{2}(q - 1)q(x^2 - 1)u^2 
            \\&- \frac{1}{6}(q - 1)q(2q - 1)x(x^2 - 3)u^3
            \\&+ \frac{1}{24}q\left(-3 + 6x^2 - x^4 - 12q^2(2 - 4x^2 + x^4) + 6q^3(2 - 4x^2 + x^4) + q(15 - 30x^2 + 7x^4)\right)u^4
            \\&+ r_x(u)u^4,
        \end{split}
    \end{equation}
    with $\lim_{u\to 0} r_x(u) = 0$.
\end{lemma}
\begin{proof}
    The claim follows from Taylor's theorem after computing the first four derivatives 
    of $f_x$ at $u = 0$. We computed the derivatives with Mathematica.
    The notebook together with the corresponding pdf file can be found in \url{https://github.com/DPBayes/subsampling-is-not-magic/tree/main/notebooks} under the names \textsc{Lemma\_A.2\_computation.\{nb,pdf\}}.
\end{proof}

\begin{lemma}\label{lemma:integral-f-taylor-approximation}
    When $x\sim \caln(0, 1)$,
    \begin{align}
        \Exp_x\left(\frac{1}{2}(q - 1)q(x^2 - 1)u^2 \right) &= 0 \\
        \Exp_x\left(-\frac{1}{6}(q - 1)q(2q - 1)x(x^2 - 3)u^3\right) &= 0 \\
        \Exp_x\left(\frac{f_x^{(4)}(0)}{4!}u^4\right) &= \frac{1}{4}(q - 1)^2q^2u^4.
    \end{align}
\end{lemma}
\begin{proof}
    It is well-known that $\Exp(x) = 0$, $\Exp(x^2) = 1$, $\Exp(x^3) = 0$ and $\Exp(x^4) = 3$.
    The first expectation:
    \begin{equation}
        \begin{split}
            \Exp_x\left(\frac{1}{2}(q - 1)q(x^2 - 1)u^2 \right)
            &= \frac{1}{2}(q - 1)qu^2\Exp_x\left(x^2 - 1\right)
            \\&= \frac{1}{2}(q - 1)qu^2\left(\Exp_x(x^2) - 1\right)
            \\&= 0.
        \end{split}
    \end{equation}
    The second expectation:
    \begin{equation}
        \begin{split}
            \Exp_x\left(-\frac{1}{6}(q - 1)q(2q - 1)x(x^2 - 3)u^3\right)
            &= -\frac{1}{6}(q - 1)q(2q - 1)u^3\Exp_x\left(x(x^2 - 3)\right)
            \\&= -\frac{1}{6}(q - 1)q(2q - 1)u^3\left(\Exp_x(x^3) - 3\Exp_x(x)\right)
            \\&= 0.
        \end{split}
    \end{equation}
    The third expectation:
    \begin{equation}
        \begin{split}
            &\Exp_x\left(\frac{f_x^{(4)}(0)}{4!}u^4\right)
            \\&= \Exp_x\left(\frac{1}{24}q\left(-3 + 6x^2 - x^4 - 12q^2(2 - 4x^2 + x^4) 
            + 6q^3(2 - 4x^2 + x^4) + q(15 - 30x^2 + 7x^4)\right)u^4\right)
            \\&= \frac{1}{24}q\left(-3 + \Exp_x(6x^2) - \Exp_x(x^4) - 12q^2\Exp_x(2 - 4x^2 + x^4) 
            + 6q^3\Exp_x(2 - 4x^2 + x^4) + q\Exp_x(15 - 30x^2 + 7x^4)\right)u^4
            \\&= \frac{1}{24}q\left(-3 + 6 - 3 - 12q^2(2 - 4 + 3) 
            + 6q^3(2 - 4 + 3) + q(15 - 30 + 21)\right)u^4
            \\&= \frac{1}{24}q\left(-12q^2 + 6q^3 + 6q\right)u^4
            \\&= \frac{1}{4}q^2\left(-2q + q^2 + 1\right)u^4
            \\&= \frac{1}{4}(q - 1)^2q^2u^4.
        \end{split}
    \end{equation}
    
\end{proof}

Lemmas~\ref{lemma:f-taylor-approximation} and \ref{lemma:integral-f-taylor-approximation}
show that the Taylor approximation of $\Exp_x(f_x(u))$ is 
\begin{equation}
    \Exp_x(f_x(u)) = \frac{1}{4}(q - 1)^2q^2u^4 + r(u)u^4
\end{equation}
if we can differentiate under the expectation. Next, we show that this is possible in 
Lemma~\ref{lemma:f-derivative-integrable-upper-bound}, which requires several preliminaries.

\begin{definition}
    A function $g(x, u)$ is a polynomial-exponentiated simple polynomial (PESP) if
    \begin{equation}
        g(x, u) = \sum_{i=1}^n P_i(x, u)e^{Q_i(x, u)}
    \end{equation}
    for some $n\in \N$ and polynomials $P_i(x, u)$ and $Q_i(x, u)$, $1\leq i \leq n$, with $Q_i(x, u)$
    being first-degree in $x$.
\end{definition}

\begin{lemma}\label{lemma:pesp-products-derivatives}
    If $g_1$ and $g_2$ are PESPs, 
    \begin{enumerate}
        \item $g_1 + g_2$ is a PESP,
        \item $g_1 \cdot g_2$ is a PESP,
        \item $\frac{\partial}{\partial u} g_1$ is a PESP.
    \end{enumerate}
\end{lemma}
\begin{proof}
    Let 
    \begin{equation}
        g_j(x, u) = \sum_{i=1}^{n_j} P_{i,j}(x, u)e^{Q_{i,j}(x, u)}
    \end{equation}
    for $j\in \{1, 2\}$. 
    (1) is clear by just writing the sums in $g_1$ and $g_2$ as a single sum. For (2),
    \begin{equation}
        \begin{split}
            g_1(x, u) \cdot g_2(x, u) 
            &= \sum_{i=1}^{n_1} \sum_{j=1}^{n_2} P_{i,1}(x, u)P_{j,2}(x, u)
            e^{Q_{i,1}(x, u)}e^{Q_{j,2}(x, u)}
            \\&= \sum_{i=1}^{n_1} \sum_{j=1}^{n_2} P_{i,1}(x, u)P_{j,2}(x, u)
            e^{Q_{i,1}(x, u) + Q_{j,2}(x, u)}
            \\&= \sum_{i=1}^{n_3} P_{i, 3}(x, u)e^{Q_{i,3}(x, u)}
        \end{split}
    \end{equation}
    since the product of two polynomials is a polynomial, and the sum
    of two polynomials is a polynomial of the same degree.

    For (3)
    \begin{equation}
        \begin{split}
            \frac{\partial}{\partial u} g_1(x, u)
            &= \sum_{i=1}^{n_1} \left(\frac{\partial}{\partial u}P_{i,1}(x, u)\right)e^{Q_{i,1}(x, u)}
            + \sum_{i=1}^{n_1} P_{i,1}(x, u)\left(\frac{\partial}{\partial u}Q_{i,1}(x, u)\right)e^{Q_{i,1}(x, u)}
            \\&= \sum_{i=1}^{n_1} P_{i, 4}(x, u) e^{Q_{i,1}(x, u)}
        \end{split}
    \end{equation}
    since the partial derivatives, products and sums of polynomials are polynomials.
\end{proof}

\begin{lemma}\label{lemma:gaussian-polynomial-exponential-is-integrable}
    When $x\sim \caln(0, 1)$, for $a > 0, b\in \R$ and $k\in \N$,
    \begin{equation}
        \Exp_x(a|x|^ke^{b|x|}) < \infty.
    \end{equation}
\end{lemma}
\begin{proof}
    \begin{equation}
        \begin{split}
            \Exp_x(a|x|^ke^{b|x|})
            &= \int_{-\infty}^\infty \frac{1}{\sqrt{2\pi}} a|x|^k e^{b|x|}e^{-\frac{1}{2}x^2} \dx x
            \\&\propto \int_{-\infty}^\infty |x|^k e^{b|x|}e^{-\frac{1}{2}x^2} \dx x
            \\&= 2\int_{0}^\infty x^k e^{bx}e^{-\frac{1}{2}x^2} \dx x
            \\&= 2\int_{0}^\infty x^k e^{-\frac{1}{2}(x^2 - 2bx)} \dx x
            \\&\propto \int_{0}^\infty x^k e^{-\frac{1}{2}(x^2 - 2bx + b^2)} \dx x
            \\&= \int_{0}^\infty x^k e^{-\frac{1}{2}(x - b)^2} \dx x
            \\&\leq \int_{-\infty}^\infty |x|^k e^{-\frac{1}{2}(x - b)^2} \dx x
            \\&< \infty
        \end{split}
    \end{equation}
    since all absolute moments of Gaussian distributions are finite.
\end{proof}

\begin{lemma}\label{lemma:f-derivative-integrable-upper-bound}
    For any $k\in \N$, $k\geq 1$, there is a function $g_k(x)$ such that 
    $|f_x^{(k)}(u)| \leq g_k(x)$ for all $u\in [-1, 1]$ and $x \in \R$, and 
    $\Exp_x(g_k(x)) < \infty$.
\end{lemma}
\begin{proof}
    We start by computing the first derivative of $f_x(u)$
    \begin{equation}
        \begin{split}
            f_x(u) &= \ln \frac{e^{-\frac{1}{2}x^2}}{qe^{-\frac{1}{2}((q - 1)u + x)^2} + (1 - q)e^{-\frac{1}{2}(qu + x)^2}}
            \\&= -\ln \left(qe^{-\frac{1}{2}((q - 1)u + x)^2} + (1 - q)e^{-\frac{1}{2}(qu + x)^2}\right) - \frac{x^2}{2}
        \end{split}
    \end{equation}
    \begin{align}
        f_x'(u) = -\frac{-q(q - 1)((q - 1)u + x)e^{-\frac{1}{2}((q - 1)u + x)^2}
        - (1 - q)q(qu + x)e^{-\frac{1}{2}(qu + x)^2}}
        {qe^{-\frac{1}{2}((q - 1)u + x)^2} + (1 - q)e^{-\frac{1}{2}(qu + x)^2}}
    \end{align}
    Since
    \begin{equation}
        e^{-\frac{1}{2}((q - 1)u + x)^2}
        = e^{-\frac{1}{2}((q - 1)^2u^2 + 2(q - 1)xu + x^2)}
        = e^{-\frac{1}{2}x^2}e^{-\frac{1}{2}((q - 1)^2u^2 + 2(q - 1)xu)}
    \end{equation}
    and
    \begin{equation}
        e^{-\frac{1}{2}(qu + x)^2}
        = e^{-\frac{1}{2}(q^2u^2 + 2qux + x^2)}
        = e^{-\frac{1}{2}x^2}e^{-\frac{1}{2}(q^2u^2 + 2qux)}
    \end{equation}
    we can write the first derivative in the following form:
    \begin{equation}
        f_x'(u) = \frac{P_1(x, u) e^{Q_1(x, u)} + P_2(x, u)e^{Q_2(x, u)}}{qe^{Q_1(x, u)} + (1 - q)e^{Q_2(x, u)}}
    \end{equation}
    for polynomials $P_1(x, u)$, $P_2(x, u), Q_1(x, u)$ and $Q_2(x, u)$, with $Q_1(x, u)$ and $Q_2(x, u)$ being 
    first-degree in $x$.

    We show by induction that further derivatives have a similar form:
    \begin{equation}
        f_x^{(k)}(u) = \frac{\sum_{i=1}^{n_k} P_{i,k}(x, u)e^{Q_{i,k}(x, u)}}{(qe^{Q_1(x, u)} + (1 - q)e^{Q_2(x, u)})^{2^{k-1}}}
    \end{equation}
    We also require the $Q$-polynomials to still be first degree in $x$, so the numerator 
    must be a PESP. The claim is clearly true for $k = 1$. If the claim is true for $k$, then 
    \begin{align}
        f_x^{(k+1)}(u) 
        &= \frac{G(x, u)\frac{\partial}{\partial u}F(x, u) - F(x, u)\frac{\partial}{\partial u}G(x, u)}
        {(qe^{Q_1(x, u)} + (1 - q)e^{Q_2(x, u)})^{2^{k}}}
    \end{align}
    where 
    \begin{align}
        F(x, u) &= \sum_{i=1}^{n_k} P_{i,k}(x, u)e^{Q_{i,k}(x, u)} \\
        G(x, u) &= (qe^{Q_1(x, u)}  + (1 - q)e^{Q_2(x, u)})^{2^{k-1}}
    \end{align}
    by the quotient differentiation rule. The denominator has the correct form, so it remains
    to show that the numerator is a PESP. $F$ is clearly a PESP, and so is $G$ due to 
    Lemma~\ref{lemma:pesp-products-derivatives}. The numerator is a sum of products of 
    $F$, $G$ and their derivatives, so it is a PESP by Lemma~\ref{lemma:pesp-products-derivatives},
    which concludes the induction proof.

    We have
    \begin{equation}
        qe^{Q_1(x, u)} + (1 - q)e^{Q_2(x, u)} \geq \min\left(e^{Q_1(x, u)}, e^{Q_2(x, u)}\right).
    \end{equation}
    We can split $\R$ into measurable subsets $A_1$ and $A_2$ such that 
    \begin{equation}
        \min\left(e^{Q_1(x, u)}, e^{Q_2(x, u)}\right) = 
        \begin{cases}
            e^{Q_1(x, u)} & x\in A_1 \\
            e^{Q_2(x, u)} & x \in A_2
        \end{cases}
    \end{equation}
    Now
    \begin{align}
        |f_x^{(k)}(u)|
        &\leq \left|\frac{\sum_{i=1}^{n_k} P_{i,k}(x, u)e^{Q_{i,k}(x,u)}}{\min\left(e^{Q_1(x,u)}, e^{Q_2(x,u)}\right)^{2^{k-1}}}\right|
        \\&= \left|I_{A_1}\frac{\sum_{i=1}^{n_k} P_{i,k}(x, u)e^{Q_{i,k}(x,u)}}{(e^{Q_1(x,u)})^{2^{k-1}}}\right|
        + \left|I_{A_2}\frac{\sum_{i=1}^{n_k} P_{i,k}(x, u)e^{Q_{i,k}(x,u)}}{(e^{Q_2(x,u)})^{2^{k-1}}}\right|
        \\&\leq \left|\frac{\sum_{i=1}^{n_k} P_{i,k}(x, u)e^{Q_{i,k}(x,u)}}{e^{{2^{k-1}} Q_1(x,u)}}\right|
        + \left|\frac{\sum_{i=1}^{n_k} P_{i,k}(x, u)e^{Q_{i,k}(x,u)}}{e^{{2^{k-1} Q_2(x,u)}}}\right|
        \\&\leq \sum_{i=1}^{m_k} |R_{i,k}(x, u)|e^{S_{i,k}(x, u)}
    \end{align}
    where $R_{i,k}(x,u)$ and $S_{i,k}(x,u)$ are further polynomials, with the $S$-polynomials being of 
    first degree in $x$.

    Since $u\in [-1, 1]$, for a monomial $ax^{k_1}u^{k_2}$
    \begin{equation}
        ax^{k_1}u^{k_2} \leq |ax^{k_1}u^{k_2}| \leq |a|\cdot |x|^{k_1}.
    \end{equation}
    Using this inequality on each monomial of $R_{i,k}(x, u)$ and $S_{i,k}(x, u)$ gives upper 
    bound polynomials of $|x|$ $\hat{R}_{i,k}(x)$ and $\hat{S}_{i,k}(x)$ such that 
    \begin{equation}
        |f_x^{(k)}(u)|
        \leq \sum_{i=1}^{m_k} |R_{i,k}(x, u)|e^{S_{i,k}(x, u)}
        \leq \sum_{i=1}^{m_k} |\hat{R}_{i,k}(x)|e^{\hat{S}_{i,k}(x)},
    \end{equation}
    with the $\hat{S}$-polynomials being first degree.

    Let 
    \begin{equation}
        g_k(x) = \sum_{i=1}^{m_k} |\hat{R}_{i,k}(x)|e^{\hat{S}_{i,k}(x)}.
    \end{equation}
    We have shown that $|f_x(u)| \leq g_k(x)$ for $u\in [-1, 1]$ and $x\in \R$.
    The integrability of $g_k(x)$ against a standard Gaussian follows from 
    Lemma~\ref{lemma:gaussian-polynomial-exponential-is-integrable},
    as we can first push the absolute value around $\hat{R}_{i,k}(x)$ to be
    around each monomial of $\hat{R}_{i,k}(x)$ with the triangle inequality,
    and then write the resulting upper bound as a sum with each term of the form 
    $a|x|^ke^{b|x|}$, with $a > 0$, $b\in \R$ and $k\in \N$.
\end{proof}

Now we can put the preliminaries together to use the Taylor approximation of 
of $\Exp_x(f_x(u))$ to find its order of convergence.
\begin{lemma}\label{lemma:integral-f-order}
    When $x\sim \caln(0, 1)$, $\Exp_x(f_x(u)) = O(u^4)$ as $u\to 0$.
\end{lemma}
\begin{proof}
    Since $u\to 0$ in the limit, it suffices to consider $u\in [-1, 1]$. 
    First, we find the fourth-order Taylor approximation of $\Exp_x(f_x(u))$.
    Lemma~\ref{lemma:f-derivative-integrable-upper-bound} allows us to differentiate under the 
    expectation four times. Then Taylor's theorem, and Lemmas~\ref{lemma:f-taylor-approximation} and 
    \ref{lemma:integral-f-taylor-approximation} give
    \begin{equation}
        \Exp_x(f_x(u)) = \frac{1}{4}(q - 1)^2q^2 u^4 + r(u)u^4
    \end{equation}
    where $\lim_{u\to 0} r(u) = 0$.
    Now
    \begin{equation}
        \begin{split}
            \lim_{u\to 0} \frac{1}{u^4}\Exp(f_x(u))
            &= \lim_{u\to 0} \frac{1}{u^4}\left(\frac{1}{4}(q - 1)^2q^2 u^4 + r(u)u^4\right)
            \\&= \frac{1}{4}(q - 1)^2q^2 + \lim_{u\to 0} r(u)
            \\&= \frac{1}{4}(q - 1)^2q^2
            \\&< \infty,
        \end{split}
    \end{equation}
    which implies the claim.
\end{proof}

\subsection{Proof of Theorem \ref{thm:subsampled-mech-convergence}}
\label{app:subsampled-mech-convergence}
\theoremsubsampledmechconvergence*
\begin{proof}
    It suffices to show
    \begin{equation}
        \sup_x T\cdot \KL(\calm_i'(x), \calm_i(x)) \to 0
    \end{equation}
    due to Pinsker's inequality and the additivity of KL divergence for products of independent
    random variables (Lemma~\ref{lemma:kl-div-properties}). 
    When $x = 0$, the two mechanism are the same, so it suffices to look at $x = 1$. 

    KL-divergence is invariant to bijections, so
    \begin{equation}
        \begin{split}
            \KL(\calm_i'(1), \calm_i(1)) 
            &= \KL\big(\caln(q, \sigma_T^2), q\caln(1, \sigma_T^2) + (1 - q)\caln(0, \sigma_T^2)\big)
            \\&= \KL\left(\caln\left(q\frac{1}{\sigma_T}, 1\right), 
            q\caln\left(\frac{1}{\sigma_T}, 1\right) + (1 - q)\caln\left(0, 1\right)\right)
            \\&= \KL\left(\caln\left(qu, 1\right), 
            q\caln\left(u, 1\right) + (1 - q)\caln\left(0, 1\right)\right)
            \\&= \KL\left(\caln\left(0, 1\right), 
            q\caln\left(u - qu, 1\right) + (1 - q)\caln\left(-qu, 1\right)\right)
        \end{split}
    \end{equation}
    where we first divide both distributions by $\sigma_T$, then set $u = \frac{1}{\sigma_T}$, and 
    finally subtract $qu$. As $\sigma^2_T = \Omega(T)$, $u = O\left(\frac{1}{\sqrt{T}}\right)$.

    From the definition of KL-divergence, $u^4 = O(\frac{1}{T^2})$ and Lemma~\ref{lemma:integral-f-order}, 
    when $x\sim \caln(0, 1)$ we have
    \begin{equation}
        \KL(\calm_i'(1), \calm_i(1)) = \Exp_x(f_x(u)) = O(u^4) = O\left(\frac{1}{T^2}\right).
    \end{equation}
    This implies 
    \begin{equation}
        \lim_{T\to \infty} T\cdot \KL(\calm_i'(1), \calm_i(1)) = 0,
    \end{equation}
    which implies the claim.
\end{proof}

\subsection{Proof of Corollary~\ref{corollary:linear-sigma-limit}}\label{app:linear-sigma-limit}
With fixed $\Delta, T, \epsilon$, the function 
$\sigma \mapsto \AO_G(\sigma, \Delta, T, \epsilon)$ is strictly 
decreasing~\citep[Lemma 7]{balleImprovingGaussianMechanism2018} and 
continuous, so it has a continuous inverse 
$\delta \mapsto \AO_G^{-1}(\delta, \Delta, T, \epsilon)$. To declutter the 
notation, we omit the $\epsilon$ and $\Delta$ arguments from $\AO_S$, 
$\AO_G$ and $\AO_G^{-1}$ in the rest of this section.

\begin{lemma}\label{lemma:aog-and-inverse-composition-behaviour}
    For $\AO_G$ and its inverse,
    \begin{enumerate}
        \item $\AO_G(\sigma, T) = \AO_G\left(\frac{\sigma}{\sqrt{T}}, 1\right)$,
        \item $\AO_G^{-1}(\delta, T) = \AO_G^{-1}(\delta, 1)\sqrt{T}$.
    \end{enumerate}
\end{lemma}
\begin{proof}
    Recall that the PLRV of the Gaussian mechanism is $\caln(\mu_1, 2\mu_1)$ with 
    $\mu_1 = \frac{\Delta^2}{2\sigma^2}$~\citep{sommerPrivacyLossClasses2019}. By Theorem~\ref{thm:plrv-composition},
    the PLRV of $T$ compositions of the Gaussian mechanism is $\caln(T\mu_1, 2T\mu_1)$. Denoting 
    $\mu_T = T\mu_1 = \frac{T\Delta^2}{2\sigma^2}$, we see that the $T$-fold composition of the Gaussian 
    mechanism has the same PLRV as a single composition of the Gaussian mechanism with standard deviation 
    $\frac{\sigma}{\sqrt{T}}$, which proves (1).
    
    To prove (2), first we have
    \begin{align}
        \AO_G(\AO_G^{-1}(\delta, 1)\sqrt{T}, T)
        = \AO_G(\AO_G^{-1}(\delta, 1), 1)
        = \delta
    \end{align}
    by applying (1) to the outer $\AO_G$. Applying $\AO_G^{-1}(\cdot, T)$ to both sides gives
    \begin{equation}
        \AO_G^{-1}(\delta, 1)\sqrt{T} = \AO_G^{-1}(\delta, T)
    \end{equation}
\end{proof}
\corollarylinearsigmalimit*
\begin{proof}
    By definition, $\AO_S(\sigma(q, T), q, T) = \delta$, so 
    Theorem~\ref{thm:subsampled-mech-ao-convergence} implies
    \begin{equation}
        \lim_{T\to \infty} \AO_S\left(\frac{\sigma(q, T)}{q}, 1, T\right) = \delta.
        \label{eq:ao-full-batch-limit}
    \end{equation}
    Since $\AO_S(\sigma, 1, T) = \AO_G(\sigma, T)$ for any $\sigma$,
    we get from Lemma~\ref{lemma:aog-and-inverse-composition-behaviour}
    \begin{equation}
        \left|\AO_G\left(\frac{\sigma(q, T)}{q\sqrt{T}}, 1\right) - \delta\right|
        = \left|\AO_G\left(\frac{\sigma(q, T)}{q}, T\right) - \delta\right|
        \to 0
    \end{equation}
    as $T \to \infty$. 
    We have
    \begin{equation}
        \AO_G^{-1}\left(\AO_G\left(\frac{\sigma(q, T)}{q\sqrt{T}}, 1\right), 1\right)
        = \frac{\sigma(q, T)}{q\sqrt{T}}
    \end{equation}
    and by Lemma~\ref{lemma:aog-and-inverse-composition-behaviour},
    \begin{equation}
        \AO_G^{-1}(\delta, 1) = \frac{1}{\sqrt{T}}\AO_G^{-1}(\delta, T) = \frac{\sigma(1, T)}{\sqrt{T}}.
    \end{equation}
    Now, by the continuity of $\AO_G^{-1}(\cdot, 1)$,
    \begin{equation}
        \frac{1}{\sqrt{T}}\left|\frac{\sigma(q, T)}{q} - \sigma(1, T)\right| \to 0
    \end{equation}
    as $T\to \infty$.

    Since $\sigma(1, T) = \Omega(\sqrt{T})$ by Theorem~\ref{thm:sigma-growth},
    \begin{equation}
        \left|\frac{\sigma(q, T)}{q\sigma(1, T)} - 1\right|
        = \frac{|\sigma(q, T)/q - \sigma(1, T)|}{\sigma(1, T)}
        \to 0
    \end{equation}
    as $T\to \infty$, which implies the claim.
\end{proof}

\subsection{$\sigma_T^2 = \Omega(T)$}\label{app:sigma-growth}

\begin{lemma}\label{lemma:subsampled-gauss-mech-prlv-exp-lower-bound}
    Let $L$ be the PLRV of a single iteration of the Poisson subsampled Gaussian mechanism. 
    Then $\Exp(L) \geq \frac{q^2}{2\sigma^2}$.
\end{lemma}
\begin{proof}
    Recall, that for Poisson subsampled Gaussian mechanism we have the dominating pair 
    $P = q \caln(1, \sigma^2) + (1-q) \caln(0, \sigma^2)$ 
    and $Q = \caln(0, \sigma^2)$. Let $f_P$ and $f_Q$ be their densities. Now the mean of the PLRV can be written as 
    \begin{equation}
        \begin{split}
            \Exp(L) &= \Exp_{t \sim P}\left[ \log \frac{f_P(t)}{f_Q(t)} \right] \\
                &= \Exp_{t \sim P}[\log f_P(t)] - \Exp_{t \sim P}[\log f_Q(t)] \\
                &= -H(P) - \left( -\frac{1}{2}\log 2\pi\sigma^2 -\frac{1}{2\sigma^2}\Exp_{t \sim P}[t^2] \right)\\
                &= -H(P) + \frac{1}{2}\left( \log 2\pi + \log \sigma^2 + \frac{1}{\sigma^2}(\sigma^2+q) \right) \\
                &= -H(P) + \frac{1}{2}\left( \log 2\pi + \log \sigma^2 + \frac{q}{\sigma^2} + 1 \right), \label{eq:plrv_mean}
        \end{split}
    \end{equation}
    where $H$ denotes the differential entropy. The entropy term is analytically intractable for the mixture 
    of Gaussians (the $P$). However, we can upper bound it with the entropy of a Gaussian with the same
    variance, as the Gaussian distribution maximises entropy among distributions with given mean and variance. 
    The variance of $P$ is $\sigma^2+q-q^2=\sigma^2 + q(1-q)$, and therefore
    \begin{equation}
        \begin{split}
            H(P) 
                &\leq H(\caln(0, \sigma^2 + q(1-q)))  \\
                &= \frac{1}{2} \left(\log 2\pi(\sigma^2 + q(1-q)) + 1 \right) \\
                &= \frac{1}{2} \left(\log 2\pi + \log(\sigma^2 + q(1-q)) + 1 \right)
        \end{split}
    \end{equation}
    Now, substituting this into \eqref{eq:plrv_mean} we get the following lower bound
    \begin{align}
        \Exp(L)
            \geq \frac{1}{2} \left( \log \frac{\sigma^2}{\sigma^2 + q(1-q)} + \frac{q}{\sigma^2}\right)
            = \frac{1}{2} \left( -\log \lp 1 + \frac{q-q^2}{\sigma^2} \rp + \frac{q}{\sigma^2}\right).
            \label{eq:plrv_mean_lower}
    \end{align}
    Since for all $x \geq -1$,
    \begin{align}
        \log (1+x) \leq x \Leftrightarrow - \log(1+x) \geq -x,
    \end{align}
    we have
    \begin{align}
        -\log \lp 1 + \frac{q-q^2}{\sigma^2} \rp \geq -\frac{q-q^2}{\sigma^2}
    \end{align}
    which gives
    \begin{align}
        \Exp(L)
            \geq \frac{1}{2} \left( -\frac{q-q^2}{\sigma^2} + \frac{q}{\sigma^2}\right)
            = \frac{q^2}{2\sigma^2}.
    \end{align}
\end{proof}

\begin{lemma}\label{lemma:subsampled-gaussian-mech-plrv-variance-bound}
    Let $L$ be the PLRV of a single iteration of the Poisson subsampled Gaussian mechanism. 
    Then $\Var(L) \leq \frac{1}{\sigma^2} + \frac{1}{4\sigma^4}$
\end{lemma}
\begin{proof}
    \begin{equation}
        \begin{split}
            \ln \frac{f_P(t)}{f_Q(t)} 
            &= \ln \frac{q \exp(-\frac{(t - 1)^2}{2\sigma^2}) + (1 - q)\exp(-\frac{t}{2\sigma^2})} {\exp(-\frac{t^2}{2\sigma^2})}
            \\&= \ln \left(q \exp(\frac{t^2 - (t - 1)^2}{2\sigma^2}) + (1 - q)\right)
            \\&= \ln \left(q \exp(\frac{2t - 1}{2\sigma^2}) + (1 - q)\right)
            \\&\leq \ln \max\left\{\exp(\frac{2t - 1}{2\sigma^2}), 1\right\}
            \\&= \max\left\{\frac{2t - 1}{2\sigma^2}, 0\right\}
        \end{split}
    \end{equation}

    Similarly, we also have
    \begin{equation}
        \ln \frac{f_P(t)}{f_Q(t)}
        \geq \ln \min\left\{\exp(\frac{2t - 1}{2\sigma^2}), 1\right\}
        = \min\left\{\frac{2t - 1}{2\sigma^2}, 0\right\}
    \end{equation}
    so 
    \begin{equation}
        \begin{split}
            \left|\ln \frac{f_P(t)}{f_Q(t)}\right|
            &\leq \max\left\{\max\left\{\frac{2t - 1}{2\sigma^2}, 0\right\},
            -\min\left\{\frac{2t - 1}{2\sigma^2}, 0\right\}\right\}
            \\&= \max\left\{\max\left\{\frac{2t - 1}{2\sigma^2}, 0\right\},
            \max\left\{-\frac{2t - 1}{2\sigma^2}, 0\right\}\right\}
            \\&\leq \left|\frac{2t - 1}{2\sigma^2}\right| 
        \end{split}
    \end{equation}

    Since $\Exp_{t\sim P}(t) = q$ and $\Exp_{t\sim P}(t^2) = \sigma^2 + q$,
    \begin{equation}
        \begin{split}
            \Var(L^{(T)}) &\leq \Exp((L^{(T)})^2) 
            \\&= \Exp_{t\sim P}\left(\left(\ln \frac{f_P(t)}{f_Q(t)}\right)^2\right)
            \\&\leq \Exp_{t\sim P}\left(\left|\frac{2t - 1}{2\sigma^2}\right|^2\right)
            \\&= \Exp_{t\sim P}\left(\left(\frac{2t - 1}{2\sigma^2}\right)^2\right)
            \\&= \frac{1}{4\sigma^4}\Exp_{t\sim P}\left(4t^2 - 4t + 1\right)
            \\&= \frac{\sigma^2 + q}{\sigma^4} - \frac{q}{\sigma^4} + \frac{1}{4\sigma^4}
            \\&= \frac{1}{\sigma^2} + \frac{1}{4\sigma^4}.
        \end{split}
    \end{equation}
\end{proof}

\begin{lemma}\label{lemma:subsampled-gaussian-mech-plrv-tail-bound}
    Let $L_T$ be the PLRV of $T$ iterations of the Poisson subsampled Gaussian mechanism,
    and let $K \in \N$. If $\sigma_T^2 = \Omega(T)$ is not true, for any $\alpha_i > 0$, $b_i > 0$ with $1 \leq i \leq K$,
    it is possible to find a $T$ such that,
    \begin{equation}
        \Pr_{s\sim L_T}(s \leq b_i) \leq \alpha_i
    \end{equation}
    holds simultaneously for all $i$.
\end{lemma}
\begin{proof}
    If $\sigma_T^2 = \Omega(T)$ is not true, 
    \begin{equation}
        \liminf_{T\to \infty} \frac{\sigma_T^2}{T} = 0.\label{eq:sigma-growth-proof-1}
    \end{equation}
    By Lemma~\ref{lemma:subsampled-gauss-mech-prlv-exp-lower-bound} and the composition theorem,
    $\Exp(L_T) \geq T\frac{q^2}{2\sigma_T^2}$. 
    By Lemma~\ref{lemma:subsampled-gaussian-mech-plrv-variance-bound}, 
    $\Var(L_T) \leq \frac{T}{\sigma_T^2} + \frac{T}{4\sigma_T^4}$.
    
    Let $k_i = \frac{1}{\sqrt{\alpha_i}}$. Choose $T$ such that 
    \begin{equation}
        \frac{Tq^2}{\sigma_T^2} - k_i\sqrt{\frac{T}{\sigma_T^2} + \frac{T}{4\sigma_T^4}} \geq b_i.
    \end{equation}
    for all $i$.
    This is possible by \eqref{eq:sigma-growth-proof-1} by choosing $\frac{T}{\sigma_T^2}$ to be large enough to 
    satisfy all the inequalities.
    Now
    \begin{equation}
        \Exp(L_T) - k_i\sqrt{\Var(L_T)} 
        \geq \frac{Tq^2}{\sigma_T^2} - k_i\sqrt{\frac{T}{\sigma_T^2} + \frac{T}{4\sigma_T^4}} 
        \geq b_i
    \end{equation}
    for all $i$, so
    \begin{equation}
        \Pr_{s\sim L_T}(s \leq b_i)
        \leq \Pr_{s\sim L_T}\left(|s - \Exp(L_T)| \geq k_i\sqrt{\Var(L_T)}\right)
        \leq \frac{1}{k_i^2}
        = \alpha_i
    \end{equation}
    for all $i$ by Chebyshev's inequality. 
\end{proof}
This means that it is possible to make $\Pr_{s\sim L_T}(s \leq b)$ arbitrarily small
for any $b$ by choosing an appropriate $T$, and to satisfy a finite number of these 
constraints simultaneously with a single $T$.

\theoremsigmagrowth*
\begin{proof}
    By Lemma~\ref{lemma:sensitivity-to-noise-exchange}, it suffices to consider $\Delta = 1$.
    To obtain a contradiction, assume that $\sigma_T^2 = \Omega(T)$ is not true.
    Let $L_T$ be the PLRV for $T$ iterations of the Poisson subsampled Gaussian mechanism.

    From \eqref{eq:plrv-to-delta},
    \begin{equation}
        \begin{split}
            \AO_S(\sigma_T, q, T, \epsilon) 
            &= \Exp_{s\sim L_T}((1 - e^{\epsilon - s})_+)
            \\&= \Exp_{s\sim L_T}(I(s > \epsilon)(1 - e^{\epsilon - s}))
            \\&= \Exp_{s\sim L_T}(I(s > \epsilon)) - \Exp_{s\sim L}(I(s > \epsilon)e^{\epsilon - s})
            \\&= \Pr_{s\sim L_T}(s > \epsilon) - \Exp_{s\sim L}(I(s > \epsilon)e^{\epsilon - s})
        \end{split}
    \end{equation}

    By choosing $b_1 = \epsilon$ and $\alpha_1 = 1 - \frac{1}{2}(\delta + 1)$ in 
    Lemma~\ref{lemma:subsampled-gaussian-mech-plrv-tail-bound}, we get
    $\Pr_{s\sim L_T}(s > \epsilon) \geq \frac{1}{2}(\delta + 1)$.

    To bound the remaining term,
    \begin{equation}
        \begin{split}
            \Exp_{s\sim L_T}(I(s > \epsilon)e^{\epsilon - s})
            &= e^\epsilon \Exp_{s\sim L_T}(I(s > \epsilon)e^{-s})
            \\&= e^\epsilon \sum_{i=0}^\infty\Exp_{s\sim L}(I(\epsilon + i < s \leq \epsilon + i + 1)e^{-s})
            \\&\leq e^\epsilon \sum_{i=0}^\infty\Exp_{s\sim L}(I(\epsilon + i < s \leq \epsilon + i + 1)e^{-\epsilon - i})
            \\&\leq \sum_{i=0}^\infty e^{-i} \Exp_{s\sim L}(I(\epsilon + i < s \leq \epsilon + i + 1))
            \\&\leq \sum_{i=0}^\infty e^{-i} \Pr_{s\sim L}(\epsilon + i < s \leq \epsilon + i + 1)
            \\&= \sum_{i=0}^K e^{-i} \Pr_{s\sim L}(\epsilon + i < s \leq \epsilon + i + 1)
            + \sum_{i=K+1}^\infty e^{-i} \Pr_{s\sim L}(\epsilon + i < s \leq \epsilon + i + 1)
            \\&\leq \sum_{i=0}^K e^{-i} \Pr_{s\sim L}(s \leq \epsilon + i + 1)
            + \sum_{i=K+1}^\infty e^{-i}.
        \end{split}
    \end{equation}
    The series $\sum_{i=0}^\infty e^{-i}$ converges, so it is possible to make $\sum_{i=K+1}^\infty e^{-i}$ arbitrarily 
    small by choosing an appropriate $K$, which does not depend on $T$.

    If we choose $K$ such that $\sum_{i=K+1}^\infty e^{-i} < \frac{1}{4}(1 - \delta)$ and then choose $T$ such that 
    $e^{-i}\Pr(s \leq \epsilon + i + 1) < \frac{1}{4(K+1)}(1 - \delta)$ for all $0 \leq i \leq K$, we have 
    \begin{equation}
        \Exp_{s\sim L_T}(I(s > \epsilon)e^{\epsilon - s}) 
        < (K+1)\cdot \frac{1}{4(K+1)}(1 - \delta) + \frac{1}{4}(1 - \delta)
        = \frac{1}{2}(1 - \delta).
    \end{equation}

    Lemma~\ref{lemma:subsampled-gaussian-mech-plrv-tail-bound} allows multiple inequalities for a
    single $T$, so we can find a $T$ that satisfies all of the $K + 2$ inequalities we have required it to 
    satisfy. With this $T$,
    \begin{equation}
        \AO_S(\sigma_T, q, T, \epsilon) 
        = \Pr_{s\sim L_T}(s > \epsilon) - \Exp_{s\sim L}(I(s > \epsilon)e^{\epsilon - s})
        > \frac{1}{2}(\delta + 1) - \frac{1}{2}(1 - \delta)
        = \delta
    \end{equation}
    which is a contradiction.
\end{proof}

\subsection{Solving $\sigma'(q)$ for $T=1$ case}
\label{sec:app-sigma-derivative}

\lemmanocompderivative*
\begin{proof}
Recall \cref{eq:delta-no-comp}: for a single iteration Poisson 
subsampled Gaussian mechanism, we have 
\begin{equation}
    \label{eq:app-delta-no-comp}
    \begin{split}
        \delta(q) = &q\Pr (Z  \geq \sigma(q) \log \lp \frac{h(q)}{q} \rp - \frac{1}{2\sigma(q)}) \\
           &-h(q) \Pr(Z  \geq \sigma(q) \log \lp \frac{h(q)}{q} \rp + \frac{1}{2\sigma(q)}),
    \end{split}
\end{equation}
where $h(q) = e^\epsilon - (1-q)$. Since $\sigma(q)$ is a function that returns
a noise-level matching any $(\epsilon, \delta)$-DP requirement for subsampling rate
$q$, $\delta'(q) = 0$. Using Mathematica, we can solve the derivative of the RHS
in \cref{eq:app-delta-no-comp} for $\sigma'(q)$ and we get
\begin{align}
    \sigma'(q) = \frac{\sqrt{\frac{\pi }{2}} \sigma(q)^2 e^{\frac{1}{2} \sigma(q)^2 \log ^2\left(\frac{q+e^{\epsilon }-1}{q}\right)+\frac{1}{8 \sigma(q)^2}} \left(\text{erf}\left(\frac{1-2 \sigma(q)^2 \log \left(\frac{q+e^{\epsilon }-1}{q}\right)}{2 \sqrt{2} \sigma(q)}\right)-\text{erf}\left(-\frac{2 \sigma(q)^2 \log \left(\frac{q+e^{\epsilon }-1}{q}\right)+1}{2 \sqrt{2} \sigma(q)}\right)\right)}{q \sqrt{\frac{q+e^{\epsilon }-1}{q}}}.
\end{align}
The notebook together with the corresponding pdf file can be found in \url{https://github.com/DPBayes/subsampling-is-not-magic/tree/main/notebooks} under the names \textsc{no\_comp\_derivative.\{nb,pdf\}}.

Note that 
\begin{align}
    &\exp(\frac{1}{2} \sigma(q)^2 \log ^2\left(\frac{q+e^{\epsilon }-1}{q}\right)+\frac{1}{8 \sigma(q)^2}) \\
        &=\exp(\frac{1}{2} \lp \sigma(q)^2 \log ^2\left(\frac{q+e^{\epsilon }-1}{q}\right)+\frac{1}{4 \sigma(q)^2} \rp ) \\
        &=\exp(\frac{1}{2} \lp \lp \sigma(q) \log \left(\frac{q+e^{\epsilon }-1}{q}\right) - \frac{1}{2 \sigma(q)} \rp^2 + \log \left(\frac{q+e^{\epsilon }-1}{q}\right) \rp ) \\
        &=\exp(\lp \frac{1}{\sqrt{2}}\sigma(q) \log \left(\frac{q+e^{\epsilon }-1}{q}\right) - \frac{1}{2\sqrt{2} \sigma(q)} \rp^2 ) \sqrt{\frac{q+e^{\epsilon }-1}{q}} \\
        &=\exp(\lp \frac{1 - 2\sigma(q)^2\log \left(\frac{q+e^{\epsilon }-1}{q}\right)}{2\sqrt{2} \sigma(q)} \rp^2 ) \sqrt{\frac{q+e^{\epsilon }-1}{q}}
\end{align}
and therefore the derivative becomes
\begin{align}
    \sigma'(q) 
    &= \frac{\sqrt{\frac{\pi }{2}} \sigma(q)^2 e^{\frac{1}{2} \lp \sigma(q) \log \left(\frac{q+e^{\epsilon }-1}{q}\right) - \frac{1}{2 \sigma(q)} \rp^2} \left(\text{erf}\left(\frac{1-2 \sigma(q)^2 \log \left(\frac{q+e^{\epsilon }-1}{q}\right)}{2 \sqrt{2} \sigma(q)}\right)-\text{erf}\left(-\frac{2 \sigma(q)^2 \log \left(\frac{q+e^{\epsilon }-1}{q}\right)+1}{2 \sqrt{2} \sigma(q)}\right)\right)}{q} \\
    &=
    \frac{\sqrt{\frac{\pi }{2}} \sigma(q)^2 e^{\lp \frac{1  - 2\sigma(q)^2 \log \left(\frac{q+e^{\epsilon }-1}{q}\right)}{2\sqrt{2} \sigma(q)} \rp^2} \left(\text{erf}\left(\frac{1-2 \sigma(q)^2 \log \left(\frac{q+e^{\epsilon }-1}{q}\right)}{2 \sqrt{2} \sigma(q)}\right)-\text{erf}\left(-\frac{2 \sigma(q)^2 \log \left(\frac{q+e^{\epsilon }-1}{q}\right)+1}{2 \sqrt{2} \sigma(q)}\right)\right)}{q}
\end{align}
Lets denote 
\begin{align}
    a &:= \frac{1}{2\sqrt{2} \sigma(q)} \\
    b &:= \frac{\sigma(q)}{\sqrt{2}} \log \lp \frac{e^\epsilon - (1-q)}{q} \rp.
\end{align}
Using this notation we can write 
\begin{align}
    q \sigma'(q) 
        &= \sqrt{\frac{\pi}{2}} \sigma(q)^2 \exp((a-b)^2) (\text{erf}(a-b) - \text{erf}(-a-b)) \\
        &= \sqrt{\frac{\pi}{2}} \sigma(q)^2 \exp((a-b)^2) (\text{erf}(a+b) - \text{erf}(b-a)).
\end{align}
Note that we have
\begin{align}
    \text{erf}'(a-b) &= \frac{2}{\sqrt{\pi}} \exp(-(a-b)^2) \\
    \Leftrightarrow 
    \exp((a-b)^2) &= \frac{2}{\sqrt{\pi} \text{erf}'(a-b)}
\end{align}
Hence
\begin{align}
    q \sigma'(q) 
        &= \sigma(q)^2 \frac{\sqrt{2}}{\text{erf}'(a-b)} (\text{erf}(a+b) - \text{erf}(b-a)) \\
        &= \sigma(q)^2 \frac{\sqrt{2}}{\text{erf}'(a-b)} (\text{erf}(a+b) + \text{erf}(a-b)) \\
        &= \sigma(q) \frac{1}{2 a} \frac{1}{\text{erf}'(a-b)} (\text{erf}(a+b) + \text{erf}(a-b)) \\
        &= \sigma(q) \frac{1}{2 a} \frac{1}{\text{erf}'(a-b)} (\text{erf}(a-b) - \text{erf}(-a-b)).
\end{align}

\end{proof}

\end{document}